\documentclass{article}

\usepackage{graphicx}
\usepackage{amsmath, amssymb, amsfonts}
 \usepackage{amsthm}
\newtheorem{theorem}{Theorem}
\newtheorem{corollary}[theorem]{Corollary}
\newtheorem{lemma}[theorem]{Lemma}

\newtheorem{proposition}[theorem]{Proposition}

\theoremstyle{definition}
\newtheorem{definition}[theorem]{Definition}
\usepackage{cite}
\usepackage{url}

\def\cX{\mathcal X}

\def\RR{\mathbb R}

\def\bE{\mathbf E}

\def\bx{\mathbf x}
\def\by{\mathbf y}
\def\bc{\mathbf c}

\def\L2p{{L^2_{\rho_{\!_\cX}}}}
\def\H{\mathcal H}
\def\calX{\mathcal X}
\def\calY{\mathcal Y}
\def\calE{\mathcal E}
\def\calH{\mathcal H}

\def\calN{\mathcal N}

\def\smath#1{\text{\scalebox{.8}{$#1$}}}

\def\sfrac#1#2{\smath{\frac{#1}{#2}}}

\graphicspath {{figures/}}
\allowdisplaybreaks

\begin{document}

\title{Optimal Rates of Distributed Regression with Imperfect Kernels}

\author{
{Hongwei Sun}\\
{\small Department of Mathematics} \\
{\small University of Jinan} \\
{\small Jinan, Shandong, China}\\
{\small Email: ss\_sunhw@ujn.edu.cn} \\
\\
{Qiang Wu} \\
{\small Department of Mathematical Sciences}\\
{\small Middle Tennessee State University}\\
{\small Murfreesboro, TN 37132, USA}\\
{\small Email: qwu@mtsu.edu}\\
}
\maketitle

\begin{abstract}
Distributed machine learning systems have been receiving increasing attentions
for their efficiency to process large scale data. Many distributed
frameworks have been proposed for different machine learning tasks. In this
paper, we study the distributed kernel regression via the divide and conquer
approach. The learning process consists of three stages. Firstly, the data is
partitioned into multiple subsets. Then a base kernel regression algorithm is
applied to each subset to learn a local regression model. Finally the local
models are averaged to generate the final regression model for the purpose of
predictive analytics or statistical inference. This approach has been proved
asymptotically minimax optimal if the kernel is perfectly selected so that the
true regression function lies in the associated reproducing kernel Hilbert
space. However, this is usually, if not always, impractical because kernels
that can only be selected via prior knowledge or a tuning process are hardly
perfect. Instead it is more common that the kernel is good enough but imperfect
in the sense that the true regression can be well approximated by but does not
lie exactly in the kernel space. We show distributed kernel
regression can still achieves capacity independent optimal rate in this case.
To this end, we first establish a general framework that allows to analyze distributed regression with
response weighted base algorithms  by bounding the error of such algorithms on a single data set,
provided that the error bounds has factored the impact of the unexplained variance of the response variable.
Then we perform a leave one out analysis of the kernel ridge regression and bias corrected kernel ridge regression,
which in combination with the aforementioned framework allows us to derive sharp error bounds and capacity independent optimal rates
for the associated distributed kernel regression algorithms.
As a byproduct of the thorough analysis, we also prove the kernel ridge  regression
can achieve rates faster than $O(N^{-1})$ (where $N$ is the sample size) in the noise free setting which, to our best knowledge, are first observed and novel in regression learning.
\end{abstract}

\section{Introduction}
\label{sec:intro}

Distributed machine learning systems have been receiving increasing attentions
for their efficiency to process large scale data. Many distributed frameworks
have been proposed for different machine learning tasks; see for instance
\cite{dean2004mapreduce, dobriban2020wonder, kraska2013mlbase, boyd2011distributed, xu2018optimal, 
zhang2015divide}. Among others, the divide and conquer approach has been proved
easy to implement but efficient for statistical estimation and predictive analytics. This approach is used when the data is too big to be analyzed by
one computer node and usually consists of three stages. First, the data is
randomly partitioned into multiple subsets. In some applications the data may
be naturally stored in different locations as a result of data collection
process and there is no need for further partitioning. Second, a base
algorithm is selected according to the learning task and applied to each subset
to learn a local model. Finally, all local models are averaged to generate the
final model. This approach is computationally efficient because the second stage can be easily parallelized. Also, because the local model training does not
require mutual communication between the computing nodes, it can largely
preserve privacy and confidentiality.

In the context of nonlinear regression analysis distributed kernel  methods
implemented via the divide and conquer approach has been widely studied and
showed asymptotically minimax optimal in many situations. In particular, if the
kernel is perfectly selected so that the true regression function lies in the
associated reproducing kernel space, the minimax optimality was verified
for kernel ridge regression
\cite{zhang2015divide,szabo2016learning,lin2017distributed},
kernel spectral algorithm \cite{GLZ16},
kernel based gradient descent \cite{lin2018distributed},
bias corrected regularization kernel network \cite{guo2017learning},
and minimum error entropy \cite{hu2019distributed}.
However, it is usually, if not always, impractical to select
perfect kernels for real world problems. More commonly one has to select an
imperfect kernel by some prior knowledge and/or a tuning process. Such a kernel
is empirically optimal within a family of candidate kernels, usually good
enough for applications, but there is no guarantee of perfectness. A typical
example is the widely used Gaussian radial basis kernel. It is effective in
most nonlinear data analysis problems because of its universality. But it is
well known that its associated reproducing kernel Hilbert space consists of
only infinitely differentiable functions. Functions that are not infinitely
differentiable can be well approximated by the kernel but cannot lie exactly in
the associated kernel space. In this situation the learning rates obtained in
the literature are suboptimal.

The primary goal of this study  is to verify the capacity independent optimality of distributed kernel regression algorithms when the kernel is imperfect. We focus on the use of  kernel ridge regression (KRR) and  bias corrected kernel  ridge regression (BCKRR) in the divide and conquer approach.
For this purpose, we propose a framework to analyze a broad class of
distributed regression methods and conduct rigorous leave one out analyses.  More specifically, we make the following four contributions.
\begin{itemize}
	\item First, we introduce the concept of response weighted regression algorithm, which covers a broad class of regression algorithms and has  both KRR and BCKRR  as examples. We propose a general framework to analyze the learning performance of response weighted distributed regression algorithms and showed that, for such algorithms, it suffices to study the learning performance of the response weighted algorithm
	on a single data set provided that it characterizes the impact of the unexplained variance of the response variable on the learning performance. This makes the analysis of such distributed learning algorithms much easier. Because both KRR and BCKRR fall into this framework, we will utilize it to analyze the distributed KRR and distributed BCKRR in this study.
	
	\item Second, we conduct a leave one out analysis of the KRR algorithm
	and prove capacity independent error bounds, which are sharp in the sense that they lead to optimal capacity independent rates regardless the kernel is perfect or imperfect.  While the idea of leave one out analysis was originally developed in \cite{bousquet2002stability, zhang2003leave}, our analysis is more rigorous so that the error bounds factor in the impact of the unexplained variance of the response variable and therefore we are able to utilize them in combination with the aforementioned framework to derive sharp error bounds and optimal learning rates for distributed KRR. Furthermore, our analysis  also greatly relaxes the restriction on the number of local machines used in the distributed regression. In particular, our results indicate that fast rates can still be achieved with the number of local machines increasing at the order of $O(\sqrt N)$ (with $N$ being the sample size) if the true regression does lie in the kernel space. This is a significant relaxation compared with those in the literature; see detailed comparison in Section \ref{subsec:comparison}.
	
	\item Third, we conduct a rigorous leave one out analysis of the BCKRR and utilize the results to derive error bounds and learning rates for distributed BCKRR algorithm. Again, the results are optimal from a capacity independent viewpoint when the kernel is imperfect. As the BCKRR was proposed by the idea of bias correction, its original formula involves a two-step procedure and admits an operator representation which, if not impossible, is unsuitable for leave one analysis in a natural way. To overcome this difficulty we prove two alternative formulae for the algorithm, among which the recentering regularization formula defines the target function by a Tikhonov regularization scheme and allows us to conduct leave one analysis naturally. Moreover, these two perspectives also shed light on the design of other bias corrected algorithms whose solutions do not have an explicit representation like KRR. See Section \ref{sec:perspects}
	and the discussions in Section \ref{sec:conclusion} for details.
	
	\item Last, as a byproduct of our leave one out analysis, we derive super fast learning rates for both KRR and BCKRR when the unexplained variance of the response variable becomes zero, that is, the response value is determined and noise free for any fixed input. If the kernel is perfect, the rate can be faster than $O(N^{-1})$ and even as fast as  $O(N^{-2})$ in the best situation. To our best knowledge, such super fast rates for kernel regression are first observed and novel  in learning theory research.
\end{itemize}

The rest of the paper is organized as follows. In Section \ref{sec:results}
we describe the problem setting, algorithms, and the main results.
Discussions and detailed comparisons between our results and those in the
literature will be given. Empirical studies will be used to illustrate
the effectiveness of imperfect kernels in distributed regression. In Section \ref{sec:framework} we propose a general framework for the analysis of response weighted distributed regression algorithms. Some preliminary lemmas were proved in Section \ref{sec:pre}. Then in Section \ref{sec:krr} and Section \ref{sec:bckrr}
 we conduct leave one out analyses of KRR and BCKRR, respectively. The results are then used to prove our main results regarding the error bounds and learning rates of distributed KRR and distributed BCKRR. We close with conclusions and discussions in Section \ref{sec:conclusion}.

\section{Problem setting and main results}
\label{sec:results}

Let $\calX$ be the sample space of
input variable $x$ and $\calY$ the sample space of the response variable $y.$
They are linked by joint a probability measure $\rho$ on the product space
$\calX\times \calY.$
The goal is to learn the mean regression function that minimizes the mean  squared prediction error, i.e.,
$$ f^* =  \arg\min \calE(f) \qquad \hbox{ where } \calE(f) = \bE[(y-f(x))^2].$$
 This is usually implemented by
minimizing the empirical mean squared error or its regularized version
when we have in hand a sampled data set $D=\{(x_i, y_i), i=1,\ldots, N\}$. If $f^*$ is linear,
multiple linear regression or the regularized methods such as ridge regression or LASSO
performs well. When $f^*$ is nonlinear,
kernel ridge regression can be used to search
a good approximating function in a suitable
reproducing kernel Hilbert space.

 Let  $K$ be a Mercer kernel, namely,
 a continuous, symmetric, and positive-semidefinite function $K: \calX \times \cal X\to \RR$.
  The inner product defined by $\langle K(x, \cdot), K(t, \cdot)\rangle_K =K(x, t)$
 induces a reproducing kernel Hilbert space (RKHS) $\H_K$ associated to the kernel $K$.
 The space is the closure of the function class spanned by
 $\{K_x=K(x, \cdot): x\in\calX\}.$
 The reproducing property $f(x) = \langle f, K_x\rangle_K$ leads to
 $|f(x)|\le \sqrt{K(x,x)} \|f\|_K$. Thus if $\kappa = \sup_{x\in\calX}\sqrt{K(x,x)}<\infty$, then
  $\calH_K$ can be embedded into $C(\calX)$ and $\|f\|_\infty \le \kappa\|f\|_K.$
We refer to  \cite{Aron} for more other properties of RKHS.
 The kernel ridge regression (KRR)
 estimates the true regression function $f^*$ by the function $f_{D, \lambda} \in \H_K$ minimizing the regularized sample mean squared error,
 \begin{equation}
 \label{eq:krr}
 f_{D, \lambda} = \arg\min_{f\in\calH_K} \left\{\frac 1 N \sum_{i=1}^N (y_i-f(x_i))^2 +\lambda \|f\|_K^2\right\},
 \end{equation}
 where $\lambda >0$ is a regularization parameter.
It is a popular kernel method for nonlinear regression analysis.
Its predictive consistency has been extensively studied in the literature;
see e.g. \cite{EPP, bousquet2002stability, zhang2003leave, devito2005model, wu2006learning, bauer2007regularization,
caponnetto2007optimal, smale2007learning,  sun2009note, steinwart2009optimal}
and many references therein.
Its applications were also extensively explored and shown successful in  many problem domains.

 By the famous representer theorem, $f_{D, \lambda}$ admits  a representation $f_{D, \lambda}(x) = \sum_{i=1}^N c_iK(x_i, x)$
 with $\bc =(c_1,\ldots, c_N)^\top\in \RR^N$ solved from the linear system $(\lambda N \mathbb I+\mathbb K)\bc = \by$
where $\mathbb I$ is the identity matrix,  $\mathbb K=[K(x_i, x_j)]_{i,j=1}^N$ is the kernel matrix on the input data and $\by=(y_1, \ldots, y_N)^\top\in \RR^N.$
Let $S:\H_K\to \RR^N$ be be the sampling operator
defined by $Sf=(f(x_1), \ldots, f(x_N))^\top$ for $f\in\H_K.$ Its dual operator $S^*$ is given by
$S^*\bc = \sum_{i=1}^N c_iK(x_i, \cdot) \in \H_K $ for $\bc\in\RR^N.$
In \cite{smale2007learning} it is proved that $f_{D, \lambda}$ has an operator representation
\begin{equation}
\label{eq:krrop}
f_{D, \lambda} = \frac 1 N \left(\lambda I + \frac 1 N S^*S\right)^{-1}S^* \by .
\end{equation}
One objective of this paper is to study the performance of the distributed version of this method.

We need several assumptions that are used throughout the paper. Recall the true regression function is $f^*(x)=\bE[y|x]$ and define
$\sigma^2 = \bE[{\rm var}(y|x)] = \bE[(y-f^*(x))^2].$  Notice that
$${\rm var}(y) = \bE\Big[{\rm var}(y|x)\Big] + {\rm var}\Big(\bE[y|x]\Big) = \sigma^2  + {\rm var}\Big(f^*(x)\Big).$$
The second term on the right is the part of variance of $y$ that is explained by the regression function while the first term is unexplained. Our first assumption is on the finiteness  of unexplained variance.

\medskip

\noindent {\bf Assumption 1.} The unexplained variance  of the response variable is finite, i.e.,
$\sigma^2<\infty.$

\medskip

If in particular $\sigma^2=0,$ then $y=f^*(x)$ is determined for each given $x\in\calX$ and we call the regression problem is noise free. Note also for any function $f$ independent of $(x,y)$ there holds
\begin{equation}
\label{eq:fdec}
\bE \left[ (y- f(x))^2 \right] =\bE\left[(f^\ast (x)-f(x))^2 \right]
+\sigma^2.
\end{equation}
 This identity will be repeatedly used in the proof of our main results.

Next we need the so called source condition. To state it, let $\rho_{\!_\calX}$ denote the marginal distribution on $\calX.$ Define
$$L_Kf(x) = \int_{\calX} K(x, t) f(t) \hbox{d}\rho_{\!_{\calX}}(t).$$
Then $L_K$ defines a compact operator both on $\L2p$
(the space of square integrable functions with respect to the probability measure $\rho_{\!_{\calX}}$)
and $\calH_K.$
Let $\tau_i$ and $\phi_i$ be the eigenvalues and eigenfunctions of $L_K$
as an operator on $\L2p$.
Then $\{\phi_i\}_{i=1}^\infty$ form an orthogonal basis of $\L2p$ and
$$L_Kf = \sum_{i=1}^\infty \tau_i\langle f, \phi_i\rangle_{\!_{\L2p}} \phi_i, \qquad \forall\ f\in \L2p.$$
Also, $\{\psi_i= \sqrt{\tau_i}\phi_i: \tau_i\not=0\}$ form an orthonormal basis of $\calH_K$
and, as an operator on $\calH_K$,
$$L_Kf = \sum_{i:\tau_i\not=0} \tau_i \langle f, \psi_i\rangle_{\!_K} \psi_i, \qquad \forall \ f\in\calH_K.$$ It is easy to verify that $L_K$ is a population version of  the operator $\frac 1 N S^*S$ (as operators on $\H_K$).

Because $L_K$ is compact and admits an eigen-decomposition form, $L_K^r$ is well defined for all $r>0.$ In particular,
let $\overline {\H_K}$ be the closure of $\H_K$ in $\L2p.$ Then for each $f\in\overline{\H_K}$ we have $L_K^{\frac 12} f\in\H_K$ and
\begin{equation}
\label{eq:iso}
\|f\|_\L2p = \|L^{\frac 12} f\|_K.
\end{equation}  If the kernel $K$ is universal in the sense that $\H_K$ is dense in $C(\calX)$, then $\overline{\H_K}=\L2p$ and \eqref{eq:iso} holds for all $f\in\L2p.$
The source condition is stated as follows.

\medskip
\noindent{\bf Assumption 2.}
There exist some $u^*\in \L2p$ and $r>0$ such that $f^*=L_K^r u^*.$

\medskip

If the source condition holds with $r\ge \frac 12$ then $f^*\in \calH_K$ and the kernel is perfect for learning the regression function. If $r<\frac 12$, then $f^*$ does not lie in $\calH_K$ and the kernel is imperfect. We will focus our study on the imperfect case in this paper.

Before moving on to the distributed KRR, let us first recall that if the source condition holds with $0<r\le 1$,  KRR can reach the optimal capacity independent  rate $O(N^{-\frac {2r} {1+2r}})$ \cite{zhang2003leave}. If $\frac 12 \le r \le 1$ and in addition the capacity of the reproducing kernel Hilbert space as measured by the effective dimension satisfies
\begin{equation}
\label{eq:effdim}
\calN(\lambda) = \hbox{Tr}((\lambda I + L_K)^{-1} L_K) \le C_0 \lambda^{-\beta}
\end{equation}
for some $C_0>0$ and $0<\beta<1,$
KRR reaches the minimax optimal capacity dependent  rate  $O(N^{-\frac {2r}{\beta+2r}})$
\cite{caponnetto2007optimal}. However, to our best knowledge, if $r<\frac 12$ meaning that the kernel is imperfect, such minimax optimality has never been verified in the literature. Note all our results in this paper will not assume any capacity conditions.

\subsection{Distributed kernel ridge regression}

In the context of distributed kernel regression we divide the whole data $D$ into $m$ disjoint subset $D= \bigcup_{\ell=1}^m D_\ell.$ Without loss of generality we assume all data sets are of equal size $n=N/m$ and denote $D_\ell = \{(x_{\ell,1}, y_{\ell, 1}), \ldots, (x_{\ell, n}, y_{\ell, n})\}.$
Let $f_{D_\ell, \lambda}$ be the local estimator learned from $D_\ell$ by using KRR method \eqref{eq:krr}. The distributed KRR defines the final global estimator by
\begin{equation}
\label{weightaverage}
\overline{f_{D,\lambda}}= \sum_{\ell=1}^m \frac nN f_{D_\ell, \lambda}
= \frac1{m} \sum_{\ell=1}^{m} f_{D_\ell, \lambda }.
\end{equation}
This approach has been studied in \cite{zhang2015divide, lin2017distributed} and the minimax optimality was verified for $\frac 12 \le r\le 1$.  In this paper we prove the following capacity independent bounds for all $r>0.$

\begin{theorem}\label{DKRRtheorem}
	Assume  $\sigma^2<\infty$ and  $f^* = L_K^{r} u^*$ for some $u^*\in L^2(P_\calX)$ and $0<r \le 1.$
	\begin{enumerate}
		\item [{\rm (i)}] If $0<r\le \frac 1 2$, then there exists
	a constant $C_1>0$ independent of $N,$ $n,$ $m$, or $\lambda$
	such that
	\begin{equation*}
	\bE\left[ \left \|\overline{f_{D,\lambda}}-f^*\right \|^2_{\L2p} \right]  \leq
	C_1 \left\{ \frac{m\sigma^2}{N^2\lambda^2}
	+ \frac{\sigma^2}{N\lambda}+
	\lambda^{2r} \bigg(1+\frac{m^2}{N^2\lambda^2}
	+ \frac{m}{N\lambda} \bigg)\right\}.
	\end{equation*}
Consequently,  if
	$m\leq N^{\frac{2r}{1+2r}},$ then with the choice  $\lambda=N^{-\frac{1}{1+2r}}$, we have
	\begin{equation*} 
	\bE\left[\left \|\overline{f_{D,\lambda}}-f^*\right \|^2_{\L2p} \right]  = O\Big(N^{-\frac{2r}{1+2r}}\Big).
	\end{equation*}
	\item [{\rm (ii)}] If $\frac 12<r\le 1$ and $\lambda n\ge 1$, then there exists a constant $C_2>0$ such that
		\begin{equation*}
	\bE\left[\left \|\overline{f_{D,\lambda}}-f^*\right \|^2_{\L2p} \right]  \leq
	C_2 \left\{ \frac{m\sigma^2}{N^2\lambda^2}
	+ \frac{\sigma^2}{N\lambda}+ \lambda^{2r}  +
	\frac{\lambda ^r  m^{\frac r 2 +\frac 1 4}} {N^{\frac r 2 +\frac 14} } \right\}.
	\end{equation*}
	Consequently,  if
	$m\leq N^{\frac{4r^2+1}{(1+2r)^2}},$ then with the choice  $\lambda=N^{-\frac{1}{1+2r}}$, we have
	\begin{equation*} 
	\bE\left[\left \| \overline{f_{D,\lambda}}-f^*\right \|^2_{\L2p} \right]  = O\Big(N^{-\frac{2r}{1+2r}}\Big).
	\end{equation*}
\end{enumerate}
\end{theorem}

Note if $r>1$ the error bounds and convergence rates are the same as $r=1$ due to the saturation effect of KRR and have been omitted. The error bounds in Theorem \ref{DKRRtheorem} are sharp and the rates are capacity independent optimal. A detailed comparison with the results in the literature are given in Section \ref{subsec:comparison}.

\subsection{Bias correction}

The bias corrected kernel ridge regression (BCKRR) was proposed in \cite{wu2017bias} to efficiently handle block wise data,
for which the distributed learning is an example.
Notice the operator representation \eqref{eq:krrop} implies
 $f_{D, \lambda}$ has asymptotic bias  $-\lambda (\lambda I + L_K)^{-1} f^*$.
The BCKRR is formulated by
subtracting an empirical estimate of the asymptotic bias:
 \begin{equation} \label{eq:fa}
f^\sharp_{D,\lambda} = f_{D, \lambda}  + \lambda \left(\lambda I + \frac 1 n S^*S\right)^{-1}  f_{D, \lambda}.
\end{equation}
Similar to distributed KRR, a distributed BCKRR can be designed by applying BCKRR  on each subset and averaging the local estimators to obtain the global estimator as
$$\overline {f^\sharp _{D,\lambda}}
	= \frac 1 m \sum_{\ell=1}^ m f^\sharp_{D_\ell, \lambda} . $$
The bias and variance of BCKRR has been characterized in \cite{wu2017bias} and
the distributed BCKRR was studied in \cite{guo2017learning}. Those studies have shown that the BCKRR benefits the block wise data analysis both theoretically and empirically. Similar to the distributed KRR case, the distributed BCKRR has been shown asymptotically minimax optimal when the kernel is perfect. But when the kernel is imperfect, Guo et al
\cite{guo2017learning} derived the rate of $O(N^{-\frac {r}{\beta+2r}})$
under the capacity condition \eqref{eq:effdim}.
It implies a capacity independent rate of $O(N^{-\frac {r}{1+2r}}),$ which is far from optimal.
Even with the capacity condition, if it is weak, say $\beta>\frac 12 - r,$
the rate in \cite{guo2017learning}  is still worse than the capacity independent rate
$O(N^{-\frac{2r}{1+2r}})$. One of our main contributions  is to address this question.

\medskip

\begin{theorem}
	\label{thm:2}
	Assume  $\sigma^2<\infty$ and  $f^* = L_K^{r} u^*$ for some $u^*\in \L2p$ with some $0< r \le 2.$ If $\lambda n \ge 1$, then there exists
	a constant $\widetilde C>0$ independent of $N,$ $n,$ $m$, or $\lambda$
	such that
	\begin{equation*}
	\bE\left[\left \|\overline {f^\sharp _{D,\lambda}} -f^*\right \|^2_{\L2p} \right]  \leq
	\widetilde C  \begin{cases}
	 \dfrac{\sigma^2}{N\lambda}+ \lambda^{2r},
	 & \hbox{ if }  0<r\le \frac 12; \\[1em]
	  \dfrac{\sigma^2}{N\lambda}+ \lambda^{2r}  +
	 \dfrac{\lambda ^r  m^{\frac r 2 +\frac 1 4}} {N^{\frac r 2 +\frac 14} },
	 & \hbox{ if } \frac 12 <r\le 1; \\[1em]
	  \lambda^{2r}
	+\dfrac{\lambda^{r} m^{\frac r2 +\frac 14}}{N^{\frac r2 +\frac 14}}  + \dfrac{\lambda^{r+\frac 12} m^{\frac 12}}{N^{\frac 12}}
	+ \dfrac {\sigma^2} {\lambda N} ,
	& \hbox{ if } 1 <r \le \frac 32; \\[1em]
	 \lambda^{2r}
	+ \dfrac{\lambda m}{N}+\dfrac{\lambda^{r+\frac 12} m^{\frac 12}}{N^{\frac 12}}
	+ \dfrac {\sigma^2} {\lambda N} ,
	& \hbox{ if } \frac 32 <r \le 2.
	\end{cases}
	\end{equation*}
	Consequently, if choosing $\lambda = N^{-\frac 1 {2r+1}}$ and the number of local machines satisfies  $m\le N^{\theta}$ with
	\begin{equation*}
	\theta = \begin{cases}
	\dfrac{2r}{1+2r}, & \hbox{ if } 0<r\le \frac 12; \\[1em]
	\dfrac{4r^2+1}{(1+2r)^2}, & \hbox{ if } \frac 12<r\le \frac{1+\sqrt{2}}{2} ; \\[1em]
	\dfrac{2}{1+2r}, & \hbox{ if } \frac{1+\sqrt{2}}{2}<r\le 2,
	\end{cases}
	\end{equation*}
 then we have
 \begin{equation*}
 \bE\left[\left \|\overline {f^\sharp _{D,\lambda}} -f^*\right \|^2_{\L2p} \right]  = O\Big(N^{-\frac{2r}{1+2r}}\Big).
 \end{equation*}
\end{theorem}

\subsection{Comparison with literature}
\label{subsec:comparison}

The minimax rate analysis of regularized least square algorithm have been extensively studied in statistics and learning theory literature. If the eigenvalues of the operator $L_K$ decays as
$\tau_i \le i^{-\frac {1}{\beta}}$, which implies the capacity condition \eqref{eq:effdim},
then the minimax learning rate of kernel ridge regression is $O(N^{-\frac{2r}{2r+\beta}})$
for $\frac 12 \le r \le 1$; see e.g. \cite{caponnetto2007optimal, lin2017distributed, blanchard2018optimal}. To our best knowledge, this rate has never been proved for $r<\frac 12$ without imposing additional conditions.

Note that the capacity condition \eqref{eq:effdim} roughly measures the smoothness
of kernel $K$. The smoother the kernel is, the smaller the $\beta.$ For example,
it is proved in  \cite{mendelson2010regularization} that,
if $K\in C^\alpha(\mathcal{X}^2\times\mathcal{X}^2)$
with some integer $\alpha\geq 1$ and $\mathcal{X}^2$
is locally the graph of a Lipschitz function, then \eqref{eq:effdim}
is satisfied with $\beta =\left( \frac{\alpha}{2\mathrm{dim}(\mathcal{X})}
+\frac12 \right)^{-1}$.  Moreover,  the Mercer theorem guarantees all kernels satisfies \eqref{eq:effdim} with $\beta=1$. Therefore, $\beta=1$ corresponds to the worst situation and imposing a capacity condition \eqref{eq:effdim} with $\beta=1$ is equivalent to no capacity condition. We see the rate $O(N^{-\frac {2r}{2r+1}})$ is minimax optimal in a capacity independent sense. It is proved in \cite{zhang2003leave} that kernel ridge regression can achieve this capacity independent optimal rate when $r\le \frac 12.$

Concerning the distributed kernel ridge regression via the divide and conquer technique,
under the assumptions that $\bE[|\phi_i(x)|^{2k}] \le A^{2k}$
for some $k>2$ and constant $A<\infty$, $\lambda_i \le a i^{-\frac 1 \beta}$, and $f^* \in \H_K$ (i.e. $r=\frac 12$), it is proved
in \cite{zhang2015divide} that the optimal learning rate of $ O(N^{- \frac{1}{1+\beta}})$
can be achieved by  restricting the number of local processors
$$m \le c_\alpha \left(\frac{N^{\frac{k-4-k\beta }{1+\beta}}}{A^{4k}\log^k N}\right)^{\frac 1 {k-2}}.$$ This unfortunately does not apply to the case $\beta=1.$
Later in \cite{lin2017distributed}  the minimax optimal rate was verified for
all $r\in[\frac 12, 1]$ if
$
m \le N ^{\min(\frac{3(2r-1)+\beta }{5(2 r +\beta )}, \frac{2r-1}{2 r +\beta })}
$
and in \cite{GLZ16} the restriction was relaxed to
$m \le N^{\min(\frac{2}{2r+\beta}, \frac{2r-1}{2r+\beta} )} .$
Although their results apply to $\beta=1$, the restriction on $m$ is strict.
If $r=\frac 12$, the data is essentially not able to be distributed according to their results.
The error bounds in \cite{lin2017distributed, GLZ16} also  apply to $r\le \frac 12$ but only lead to  suboptimal rate $O(N^{-\frac{2r}{1+\beta}}) $ with the restriction $m=O(1)$. If $\beta=1$, the capacity independent rate is $O(N^{-r}).$

In this paper our primary interest is to understand the performance of distributed kernel regression when the kernel is imperfect, i.e., when $r<  \frac 1 2$ and thus
the true regression function $f^*$ is not in $\H_k.$ For this case,
Theorem \ref{DKRRtheorem} (i) tells that, with the restriction
$m\le N^{\frac {2r}{1+2r}},$  the capacity independent optimal rate
$O(N^{-\frac {2r}{1+2r}})$ for $r\le \frac 12$ can be achieved.
Note this rate  is even faster than
the existing capacity dependent rate $O(N^{-\frac{2r}{1+\beta}})$
when $r<\frac \beta 2$,
let alone the capacity independent rate  $O(N^{-r}).$ In other words,
we proved faster rates under weaker restrictions on $m.$ More importantly,
unlike previous results that requires $m=O(1)$ and thus disclaims the effectiveness of  distributed kernel regression algorithms in this situation,
our results instead verified their feasibility.

If $f^*$ lies in $\H_K$, i.e. $r\ge \frac 12$, it is as expected that the rate $O(N^{-\frac {2r}{1+2r}})$ in Theorem \ref{DKRRtheorem} is worse than the minimax capacity dependent rate in the literature. But comparing the restrictions on the number of local machines, we see our restriction is much relaxed. In particular, when $r$ is close to $\frac 12$, the restrictions in the literature approaches to $O(1)$. But our restriction still allows about $O(\sqrt N)$ local machines while preserving fast rates.
It is a sacrifice of convergence rates for more local machines and may
be useful for applications where analysis of super big data is necessary.

Concerning the distributed BCKRR, similar conclusions can be made. When $r<\frac 12$,
the rate in Theorem \ref{thm:2} is faster and the restriction is more relaxed than
the results in \cite{guo2017learning} while  when $r\ge \frac 1 2$ the rate is worse than those in \cite{guo2017learning} but the restriction on $m$ is greatly relaxed.
As noted in \cite{wu2017bias, guo2017learning} a main theoretical advantage
of BCKRR is to relax the saturation effect of KRR. Comparing Theorem \ref{thm:2}
with Theorem \ref{DKRRtheorem}, we see the the rate of distributed BCKRR continues improving beyond $r>1$ and ceases to improve until $r=2.$ The restriction on $m$ also continue to be relaxed to $N^{\frac{\sqrt{2}}{1+\sqrt{2}}}$ until $r=\frac{1+\sqrt{2}}{2}$.

\subsection{Empirical effectiveness of   imperfect kernels} \label{subsec:experiment}

In this subsection we illustrate the empirical effectiveness of using imperfect kernels in distributed regression. To this end, we adopt the example used in \cite{zhang2015divide, guo2017learning}.
The true regression function is given by
$f^*(x) = \min(x, 1-x)$ with $x\,{\sim}\,\hbox{Uniform}[0, 1]$
and the observations are generated by the additive noise model
$y_i = f^*(x_i) + \epsilon_i$ where
$\epsilon_i\,{\sim}\,N(0, \sigma^2)$  and  $\sigma^2 =\frac 15.$
We consider two kernels: the Sobolev space kernel
$K_S(x, t) = 1+\min(x, t)$ and the
Guassian kernel $K_G(x, t) = \exp(-(x-t)^2/0.3).$
Recall  that $f^*$ belongs to $\H_{K_S}$ with  $\|f^*\|_{K_S}=1$.
So $K_S$ is a perfect kernel for this problem.
Notice that the reproducing kernel Hilbert space associated to a Gaussian kernel
consists only  infinitely differentiable functions and even polynomials may not lie in the space \cite{minh2010some}. We conclude $f^*$ does not lie in $\H_{K_G}$ and
$K_G$ is an imperfect kernel.
We generate $N = 4098$ sample points and use number of partitions
$m\in\{2, 4, 8, 16, 32, 64, 128, 256, 512, 1024\}.$ Mean squared errors between the estimated function and the true regression function is used to measure the performance.

According to \cite{zhang2015divide}, if the Soboleve space kernel $K_S$ is used,  the theoretically optimal choice of the regularization parameter is $\lambda = N^{-\frac 23}$.
When Guassian kernel is used, by Theorem \ref{DKRRtheorem} and Theorem \ref{thm:2}, the optimal choice of $\lambda$ should depend on the index $r$ in the source condition, which,  unfortunately, is unknown. By $0<r<\frac 12 $ we know the optimal choice should be $N^{-\alpha}$ with $\frac 12 < \alpha < 1$ and $\alpha =\frac 23$ seems an acceptable choice. So we will also use $\lambda = N^{-\frac 23}$ for the Guassian kernel
to make the first comparison between the four distributed kernel regression algorithms, namely, distributed KRR with Sobolev space kernel (DKRR-S), distributed BCKRR with Sobolev space kernel (DBCKRR-S),
distributed KRR with Guassian kernel (DKRR-G),  and
distributed BCKRR with Gaussian kernel (DBCKRR-G).
To do this,  for each  aforementioned $m$ value and each algorithm,
we repeat the experiment 50 times and report the mean squared errors in  Figure~\ref{fig:SvsG}(a).
The results indicate that, even if $K_G$ is imperfect for the problem, it performs comparable with the perfect kernel $K_S$ when $m$ is small and may even outperforms $K_S$ when $m$ becomes large.

To our best knowledge, all rate analysis literature of distributed kernel regression, including \cite{zhang2015divide, lin2017distributed, lin2018distributed, guo2017learning, GLZ16} and this study,  suggest
the optimal regularization parameter be selected as  $\lambda = N^{-\alpha}$
with $\alpha$ an index depending on the regularity of the true regression function $f^*.$ While this is very helpful for researchers to understand the optimality of the algorithms, it is less informative for their practical use because such an index is unknown for real problems. Furthermore, distributed kernel regression becomes necessary only if the data is too big to be processed by a single machine. In this situation, it is imaginable that
globally tuning the optimal parameter is either impossible or too time consuming.
At the same time, note that distributed kernel regression requires underregularization, meaning that the regularization parameter must be chosen according to the total sample size $N$, not on the local sample size $n$.
To resolve all these problems, \cite{guo2017learning} proposed a practical strategy to tune the parameter for distributed kernel regression. It first
cross-validates the regularization parameter locally to get optimal choice $\lambda_{\ell, n}= n^{-\alpha_\ell}$  for each subset for  $D_\ell$ and then use an underregularized  parameter
\begin{equation}
\label{eq:underlambda}
\lambda_\ell = \lambda_{\ell,n} ^{\frac{\log N}{\log n}}=N^{-\alpha_\ell}
\end{equation}
to train the local model on $D_\ell.$
To mimic a real problem, now let us assume we have no information about the true regression function and the regularization parameters have to be selected by using the above tuning strategy.
We run the four algorithms again for each $m$ value and report the mean squared errors in Figure~\ref{fig:SvsG}(b).
We see that, when a theoretically optimal regularization parameter is not available and $\lambda$ has to be tuned from data, the performance of all four methods deteriorates faster as $m$ increases. But Guassian kernel, though imperfect, seems less sensitive to the number of local machines and so the performance deteriorates slower than Sobolev space kernel.

Finally, the results in both plots indicate that, regardless the choices of kernels and regularization parameters, bias correction always helps to improve the learning performance and relax the restriction on the number of local machines.

\begin{figure}[ht]
	\caption{Mean squared error of distributed kernel regression with Sobolev space kernel and Guassian kernel when (a) the regularization parameters is fixed as $\lambda = N^{-\frac 23}$ and (b) the regularization parameter is locally tuned and underregularized according to equation \eqref{eq:underlambda}.
	\label{fig:SvsG}}
	\includegraphics[width=0.48\textwidth]{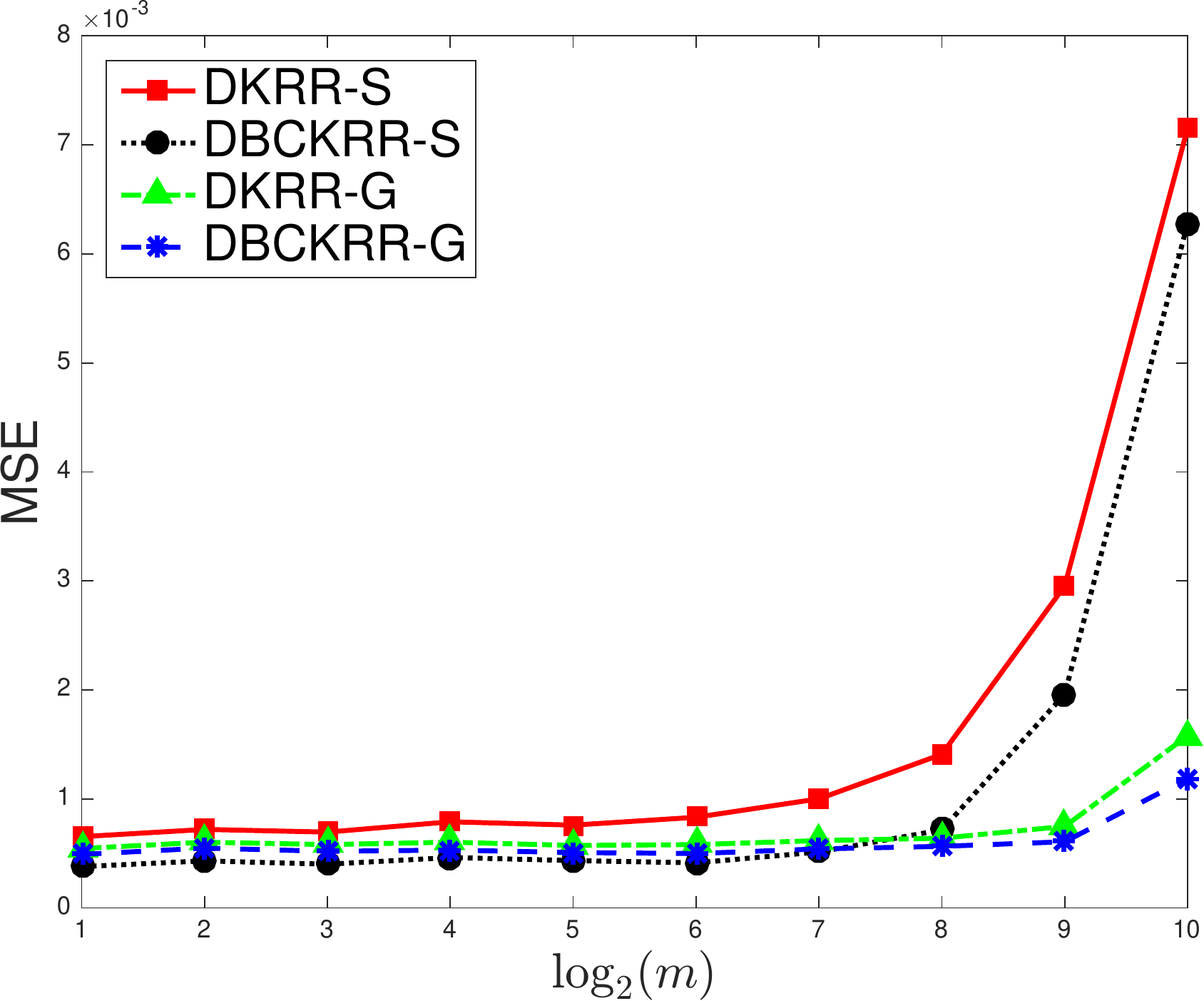}
		\includegraphics[width=0.51\textwidth]{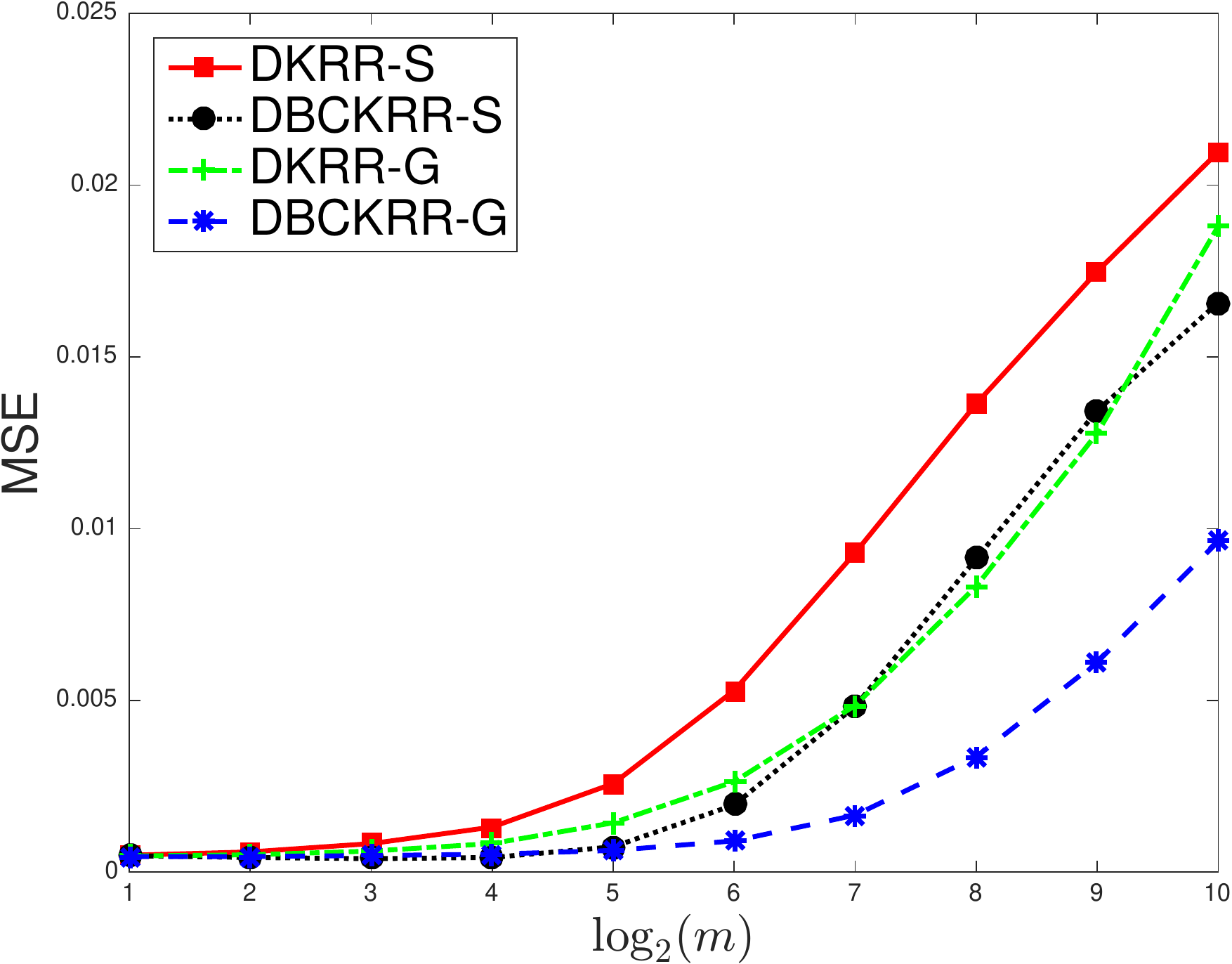}
		
	{ }	\hfill (a) \hfill \hfill  (b) \hfill { }
\end{figure}

\section{A framework for fast rate analysis of distributed regression}
\label{sec:framework}

In this section we first establish a general procedure to prove fast rate for
a class of regression algorithms that possess certain special features
and apply to the two algorithms we study in this paper.

\begin{definition}
 Let $\Lambda$ be a parameter space and $\calH$ a set of hypothesis functions.  A regression algorithm  $\mathcal A: (\calX\times \calY)^N \times  \Lambda \to \calH$ that tunes parameters in $\Lambda$ is \emph{response weighted} if for any data $D\in (\calX\times \calY)^N$ and any $\lambda \in \Lambda$ there exists a vector of  functions $g_1(D_x, \lambda), \ldots, g_N(D_x, \lambda)$ which depend only on the input data $D_x = \{x_1, \ldots, x_N\}$ and the parameter $\lambda$ but not on the output data $D_y= \{y_i\}_{i=1}^N$
 such that
 $$\mathcal A(D, \lambda ) = \sum_{i=1}^N y_i g_i(D_x, \lambda).$$
\end{definition}

There are many regression algorithms belonging to the response weighted family. By the representation of $f_{D,\lambda}$ and $f^\sharp_{D,\lambda}$, it is easy to verify that both KRR  and BCKRR belong to this family. There are more examples in the literature. For instance, traditional multiple linear regression, the kernel smooth estimators, the stochastic gradient descent algorithms associated with the least square loss are all response weighted regression algorithms.

Denote by $\overline {\mathcal A}$ the distributed algorithm that applies the divide and conquer approach and  uses $\mathcal A$ as the base algorithm on local subsets. It is easy to derive the following lemma.

\begin{lemma}
 If $\mathcal A$ is a response weighted regression algorithm, then $\overline {\mathcal A}$ is also a response weighted regression algorithm.
\end{lemma}

This lemma simply says that the response weighted feature can be inherited by the distributed algorithm. This property allows to derive fast rates for such distributed algorithms by studying error bounds of the base algorithm, as shown in the following theorem.

\begin{theorem}
	\label{thm:framework}
	Let  $\mathcal A$ be a response weighted regression algorithm and, when applying to a data $D$ of $N$ observations with parameter $\lambda,$ have the error bound as
	$$ \bE \left [ \| \mathcal A(D, \lambda ) - f^*\|_\L2p^2\right] \le \varepsilon(N, \lambda, \sigma^2).$$
	Then the corresponding distributed algorithm $\overline{\mathcal A}$ will have the error bound
	$$\bE \left [ \|\overline {\mathcal A}(D, \lambda ) - f^*\|_\L2p^2\right] \le \frac{\varepsilon(n, \lambda, \sigma^2)} {m} + \varepsilon(n, \lambda, 0) $$
	 assuming the whole data $D$ is equally split into $m$ subsets with each subset containing $n$ observations.
\end{theorem}

\begin{proof}
Let $y^*_{\ell, i} = f^*(x_{\ell,i})$ for $1\le \ell \le m,$ $1\le i\le n$ and
$$D_\ell^*= \{(x_{\ell, 1}, y_{\ell, 1}^*), \ldots, (x_{\ell, n}, y_{\ell, n}^*) \}$$
be the noise free observations associated to the subset $D_\ell$.
We have $$(\sigma^*)^2 = \bE\Big[\hbox{var}\left (y_{\ell,i}^*|x_{\ell, i}\right)\Big] = \bE\Big[\left (y_{\ell, i}^* -f^*(x_{\ell,i})\right )^2\Big] =0$$
and therefore \begin{equation}
\bE\left[ \Big \| \mathcal A(D_\ell^*,  \lambda) -f ^*\Big\|_\L2p^2\right] \le \varepsilon(n, \lambda, 0).
\label{eq:bdnonoise}
\end{equation}
Because of the response weighted feature of $\mathcal A$ we have
$$\bE\Big[ \mathcal A(D_\ell, \lambda ) | D_{\ell, x} \Big]
=\sum_{i=1}^n \bE[y_{\ell,i}|x_{\ell, i}] g_{i}(D_{\ell, x}, \lambda)
= \sum_{i=1}^n f^*(x_{\ell, i}) g_{i}(D_{\ell, x}, \lambda)
= \mathcal A(D_\ell^*,  \lambda).$$
Therefore $ \bE\big[ \mathcal A(D_\ell, \lambda )\big] = \bE\big[\mathcal A(D_\ell^*,  \lambda) \big]$ and we have
\begin{align*}
& \bE \left [ \|\overline {\mathcal A}(D, \lambda ) - f^*\|_\L2p^2\right]  \\[1em]
=\ & \bE \left [  \left\| \frac 1 m \sum_{\ell =1}^ m \Big ( \mathcal A (D_\ell, \lambda ) - \bE\left[\mathcal A (D_\ell, \lambda ) \right] \Big) \right    \|_\L2p^2 \right]  +  \left\|  \frac 1m \sum_{\ell =1}^m
 \Big( \bE\left[\mathcal A (D_\ell, \lambda ) \right] -f^*\Big) \right\|_\L2p^2 \\[1em]
\leq \ &  \frac 1 {m^2}  \sum_{\ell =1 }^m  \bE \left [  \Big\|  \mathcal A (D_\ell, \lambda ) - \bE\left[\mathcal A (D_\ell, \lambda ) \right]  \Big    \|_\L2p^2 \right] + \frac 1m \sum_{\ell =1}^m   \Big\|
 \ \bE\left[\mathcal A (D_\ell, \lambda ) \right] -f^*\Big\|_\L2p^2  \\[1em]
=\ &  \frac 1 {m^2}  \sum_{\ell =1 }^m  \bE \left [  \Big\|  \mathcal A (D_\ell, \lambda ) - f^* \Big\|_\L2p^2 \right ] + \left(\frac 1 m-\frac 1 {m^2} \right) \sum_{\ell=1}^m \Big\|  \bE\left[\mathcal A (D_\ell, \lambda ) \right]  - f^* \Big    \|_\L2p^2    \\[1em]
\le\ & \frac{1}{m^2} \sum_{\ell=1}^m \varepsilon(n, \lambda, \sigma^2)
+ \frac 1{m}  \sum_{\ell=1}^m \bE \left[ \Big\|  \mathcal A( D_{\ell}^*, \lambda )  - f^* \Big    \|_\L2p^2\right]     \\[1em]
\le\ & \frac 1{m} \varepsilon(n, \lambda, \sigma^2) +
\varepsilon(n, \lambda, 0),
\end{align*}
where we used \eqref{eq:bdnonoise} in the last step.
\end{proof}

Theorem \ref{thm:framework} states that, if the base algorithm $\mathcal A$ admits sharper error bounds for the noise free data (i.e. $\sigma^2=0$), then the distributed algorithm will be able to converges fast in the sense that the distributed computing is playing a rule and produces an estimator better than the algorithm on a single subset.

In order to use this general framework to derive error bounds and fast rates for distributed algorithm, however, we must study the performance of the base algorithm on a single data with full exploration of the impact on the unexplained variance $\sigma^2$. Although both KRR and BCKRR have already been studied in the literature, the error bounds that clearly involve $\sigma^2$ and allow us to derive error bounds of their distributed version with imperfect kernels are not available. They will be our main objectives in Sections \ref{sec:krr} and \ref{sec:bckrr}.

\section{Preliminaries}
\label{sec:pre}

In this section we provide some notations and preliminary lemmas that will be used in the proofs. For any regularization parameter $\lambda>0$, define \begin{equation}\label{flambda}
f_\lambda = \arg\min_{f\in \H_K} \left \{ \|f-f^\ast\|^2_{L_{\rho_X}^2}+\lambda\|f\|_K ^2  \right \}.
\end{equation}
It is a sample limit version of the KRR estimation $f_{D, \lambda}$
and admits an operator representation
$$ f_\lambda = (\lambda I + L_K)^{-1} L_K f^*.$$
It plays an essential role to characterize the approximation error of
the two algorithms under study. The following lemma is well known.
\begin{lemma}
	\label{lem:approx}
	Under the source condition $f^*=L_K^r u^*$ for some $r> 0$ and
	$u^*\in \L2p$, there hold
	$$\|f_\lambda - f^*\|_\L2p^2 \le \lambda^{\min(2r, 2)} \|u^*\|_\L2p^2.$$
	and
	$$ \lambda \|f_\lambda \|_K^2   \le \lambda ^{\min(2r, 1)} \|u^*\|_\L2p^2.$$
\end{lemma}


In the analysis of BCKRR, we will need the sample limit of $f_{D, \lambda}^\sharp$ defined by
\begin{equation}\label{eq:flsharp}
f_\lambda^\sharp =
f_\lambda + \lambda(\lambda I +L_K)^{-1} f_\lambda
= (\lambda I + L_K)^{-2} (2\lambda L_K + L_K^2) f^*.
\end{equation}
For this function, we have the following approximation property \cite{guo2017learning}.
\begin{lemma}
	\label{lem:approx2}
	Under the source condition $f^*=L_K^r u^*$ for some $r>0$ and
	$u^*\in \L2p$, there holds
	$$\|f_\lambda^\sharp - f^*\|_\L2p^2 \le \lambda^{\min(2r,4)} \|u^*\|_\L2p^2.$$
\end{lemma}


The following lemma can be derived from operator monotone property; see e.g. \cite{sun2009note}.
\begin{lemma}
	\label{lem:op}
	If $T_1$ and $T_2$  are two bounded self-adjoint positive operators,
	then for any $\alpha \in [0, 1]$ there holds
	$$\|T_1^\alpha -T_2^\alpha\| \le \|T_1 - T_2\|^\alpha.$$
\end{lemma}

The Lemma \ref{lem:var} below follows from simple calculation.
\begin{lemma}
	\label{lem:var}
	Let $\xi$ be a random variable with values in a Hilbert space $\mathcal H$  and $\{\xi_1,
	\ldots, \xi_N \}$ be a set of i.i.d. observations for $\xi$.
	Then
	$$\bE\left[ \left \|\frac 1 N \sum_{i=1}^N \xi_i - \bE \xi\right\|^2 \right]
	\le \frac {\bE[\|\xi\|^2]}{N}.$$	
\end{lemma}

Let $\hbox{HS}(\mathcal H_K)$ be the class of Hilbert-Schmidt operators on $\H_K$. It is known $\hbox{HS}(\mathcal H_K)$ forms a Hilbert space with the Hilbert-Schmidt norm which, give an operator $A$  on $\H_K,$  is defined by  $\|A\|_{\hbox{\scriptsize HS}}^2 = \hbox{Trace}(A^*A).$ Then $L_K$ as an operator on $\H_K$ belongs to $\hbox{HS}(\mathcal H_K)$ and $\|L_K\|_{\hbox{\scriptsize  HS}} \le \kappa^2$; see e.g. \cite{sun2009application}.
 Define the rank one operator $K_x\otimes K_x$ by
 $$  K_x\otimes K_x f = \langle f, K_x\rangle_K K_x = f(x)K_x.$$
 Then $\|K_x\otimes K_x \|_{\hbox{\scriptsize HS}} \le \kappa^2.$
 Note for any all Hilbert-Schmidt operator $A$ there holds
 $\|A\|\le \|A\|_{\hbox{\scriptsize  HS}}$.
 Applying Lemma \ref{lem:var} to $\xi = K_x\otimes K_x$  we have the following lemma.
 \begin{lemma}
 	\label{lem:S}
 	We have
 	$$\bE \left[ \left\| \frac 1 N S^*S -L_K\right\|^2 \right] \le \frac{\kappa^4}{N} $$
 \end{lemma}

\medskip

Define $$ Q(\lambda,\,N) = \min_{f\in \H_K}\left \{ \|f-f^\ast\|_\L2p^2+ \lambda \|f\|_K^2\right\}  - \lambda  \bE\left[ \| f_{D, \lambda}\|_K^2\right]   .$$
Note that  $Q(\lambda, N)$ depends on the regularization parameter $\lambda$ and sample size $N$, not on the data $D$. We will need a sharp bound for $Q(\lambda, N).$
\begin{lemma}
	\label{lem:Qbound}
Assume $f^*=L_K^r u^*$ for some $r > 0$,
$u^*\in \L2p$.
	\begin{enumerate}
		\item [{\rm (i)}] If $0< r\le \frac 12,$ then
		$$ Q(\lambda, N) \le 2 \lambda^{2r} \|u^*\|_\L2p^2. $$
		\item[{\rm (ii)}]  If $\frac 12 < r \le 1$, then
		$$ Q(\lambda, N) \le \lambda^{2r} \|u^*\|_\L2p^2 \left(
		1+ \frac {2\kappa^{2r} }{ \lambda ^r  N^{\frac r 2 +\frac 14}}
		+ \frac{2\kappa}{\sqrt{\lambda N}} \right).$$
		\item[{\rm (iii)}]  If $1<  r \le \frac 3 2$, then
		$$ Q(\lambda, N) \le \lambda^{2} \|u^*\|_\L2p^2 \left(
		1+ \frac {2\kappa^{2r} }{ \lambda   N^{\frac r 2 +\frac 14}}
		+ \frac{2\kappa \lambda^{r-1}}{\sqrt{\lambda N}} \right).$$
		\item[{\rm (iv)}]  If $r>  \frac 3 2$, then
		$$ Q(\lambda, N) \le \lambda^{2} \|u^*\|_\L2p^2 \left(
		1+ \frac {2\kappa^{3} }{ \lambda   N}
		+ \frac{2\kappa }{\sqrt{ N}} \right).$$
 	\end{enumerate}
\end{lemma}

\begin{proof} By the definition \eqref{flambda} we see
	\begin{equation}
	\label{eq:Q-1}
	Q(\lambda, N) = \|f_\lambda -f^*\|_\L2p ^2 + \lambda \|f_\lambda\|_K^2
	-\lambda \bE \left[ \|f_{D, \lambda}\|_K^2\right].
	\end{equation}
	We see conclusion (i) follows easily from Lemma \ref{lem:approx}.
	
	To prove conclusions (ii)-(iv), the key is to bound $\|f_\lambda\|_K^2
	- \bE \left[ \|f_{D, \lambda}\|_K^2\right].$
	To this end, first note that
	\begin{align*}  \|f_\lambda\|_K^2 - \bE \left[ \|f_{D, \lambda}\|_K^2\right]
	& =  \bE \left[ \left\langle f_\lambda - f_{D, \lambda}, \
	f_\lambda + f_{D, \lambda} \right\rangle _K \right]  \\
	& = 2 \bE \left[ \left\langle f_\lambda - f_{D, \lambda}, \
	f_\lambda  \right\rangle _K  - \|f_\lambda - f_{D, \lambda}\|_K^2 \right] \\
	& \le   2 \bE \left[ \left\langle f_\lambda - f_{D, \lambda}, \
	f_\lambda  \right\rangle _K  \right] \\
	& =  2   \left\langle f_\lambda - \bE[f_{D, \lambda}], \
	f_\lambda  \right\rangle _K   .
	\end{align*}
	By the operator representation \eqref{eq:krrop} of $f_{D,\lambda}$ we have \begin{equation*}
\bE[f_{D, \lambda}] = \bE\left[ \frac 1 N \left( \lambda I +
	\frac 1 N S^*S\right)^{-1} S^*S f^*\right] .
\end{equation*}
It is easy to verify that $	f_{\lambda}- \bE[f_{D, \lambda}]  =\left(\lambda I+\frac{1}{N}S ^\ast S \right)^{-1} U$ with $U= \big( \frac 1N S^*S -L_K\big)(f_\lambda-f^*).$
Applying Lemma \ref{lem:var} to the $\H_K$-valued random variable
$\xi= (f_\lambda(x)-f^*(x)) K_x$ and by Lemma \ref{lem:approx}  we have
$$ \bE\left[\left \| U\right \|_K^2\right]
\le \kappa^2 \|f_\lambda-f^*\|_\L2p^2 N^{-1}
\le {\kappa^2\|u^*\|_\L2p^2 \lambda ^{\min(2r, 2) } } {N^{-1}}.$$
Then, by the assumption  $f^* = L_K^r u^*$ for $\frac 12 <r\le \frac 32$ and using Lemma \ref{lem:op},
H\"older's inequality, and Lemma \ref{lem:S}, we obtain
	 		\begin{align}
	 	\left\langle f_\lambda - \bE[f_{D, \lambda}], \
	 	f_\lambda  \right\rangle _K  &
	 	=   \bE \left[ \left\langle  \Big(\lambda I+\frac{1}{N}S ^\ast S \Big)^{-1} U, \
	 	\Big (\lambda I + L_K\Big )^{-1} L_K^{1+r} u^*\right\rangle_K \right]
	 	\nonumber  \\[0.5em]
	 	& = \bE  \left[ \left \langle  L_K^{r+\frac{1}{2}}(\lambda  I+L_K)^{-1}
	 	\Big(\lambda I+\frac{1}{N}S^\ast S\Big)^{-1} U,\, L_K^{\frac{1}{2}} u^*  \right\rangle_K  \right] \nonumber \\[0.5em]
	 	&\leq  \|u^*\|_\L2p  \bE \left[ \left\|  L_K^{r-\frac  12}
	 	\bigg(\lambda I+\frac{1}{N}S^\ast S\Big)^{-1} U  \right\|_K\right]
	 	\label{eq:Q-2}\\[1em]
	 	& \le \|u^*\|_\L2p  \bE \left[ \Big( \lambda^{-1} \left\|  L_K^{r-\frac  12}
	 	- \Big(\frac 1 N S^*S \Big)^{r-\frac 12} \right\| + \lambda^{r-\frac 32}  \bigg)
	 	\|  U \|_K\right]\nonumber  \\[0.5em]
	 	& \le \|u^*\|_\L2p  \bE \left[ \Big( \lambda^{-1} \left\|  L_K
	 	- \frac 1 N S^*S  \right\|^{r-\frac 12} + \lambda^{r-\frac 32}  \bigg)
	 	\|  U \|_K\right] \nonumber \\[0.5em]
	 	&\le  \|u^*\|_\L2p  \left\{ \lambda^{-1} \left( \bE \left[ \left\|  L_K
	 	- \sfrac 1 N S^*S \right\|^2\right] \right) ^{\frac r2-\frac 14}  + \lambda^{r-\frac 32}  \right\}
	 	\Big(\bE \left[\|  U \|_K^2\right]\Big)^{\frac 12}  \nonumber \\[0.5em]
	 	& \le  \|u^*\|_\L2p^2 \left( \frac {\kappa^{2r} \lambda^{\min(r-1, 0)} }{ N^{\frac r 2 +\frac 14}}
	 	+ \frac{\kappa \lambda^{\min(2r-\frac 3 2, r-\frac 12)}}{N^{\frac 12}} \right).
	 	\nonumber
	 	\end{align}
Plugging this estimate into \eqref{eq:Q-1} and applying Lemma \ref{lem:approx} again we obtain the desired bounds in conclusion (ii) and conclusion (iii).

If $r > \frac 32,$ then by \eqref{eq:Q-2}, we have
	 		\begin{align*}
\left\langle f_\lambda - \bE[f_{D, \lambda}], \
f_\lambda  \right\rangle _K
	 	& \le \|u^*\|_\L2p  \bE \left[ \Big( \lambda^{-1} \left\|  L_K
- \frac 1 N S^*S  \right\| +1  \bigg)
\|  U \|_K\right]\nonumber  \\[0.5em]
& \le  \|u^*\|_\L2p^2 \left( \frac {\kappa^{3} }{ N}
+ \frac{\kappa  \lambda}{N^{\frac 12}} \right).
\nonumber
\end{align*}
The conclusion (iv) follows by plugging this estimate into \eqref{eq:Q-1}.
\end{proof}

\medskip

Next we define
$$Q^\sharp(\lambda, N )
=\|f_\lambda^\sharp - f^*\|_\L2p^2 + \lambda
\left( \bE\left[ \|f_\lambda^\sharp -f_{D, \lambda} \|_K^2 \right] -
\bE \left[ \|f_{D, \lambda}^\sharp - f_{D, \lambda}\|_K^2 \right] \right).$$
It will be used to derive the error bound for BCKRR when $r \ge 1.$

\begin{lemma}
	\label{lem:Qsbound}
	Assume  $f^*=L_K^r u^*$ for some $\frac 1 2 \leq r \leq \frac 3 2$ and
	$u^*\in \L2p$. If $\lambda\le 1$,  then
\begin{equation*}
Q^\sharp(\lambda, N )  \le
 20(\kappa^3+\kappa)\|u^*\|_\L2p^2  \left( \lambda^{\min(2r,4)}+ \frac{\lambda ^{\min(2r-1, 1)} }{N}
+ \frac{\lambda^{\min(r, 1)} }{N^{\min(\frac r2+\frac 14, 1)}}
+ \frac{ \lambda^{\min(2r-\frac 12, r+\frac 12, 2)} }{N^{\frac 12}}\right)
+ \frac{4\kappa^2\sigma^2}{\lambda N}.
\end{equation*}
\end{lemma}

\begin{proof} First note that
	$$
 \bE\left[ \|f_\lambda^\sharp -f_{D, \lambda} \|_K^2 \right] -
\bE \left[ \|f_{D, \lambda}^\sharp - f_{D, \lambda}\|_K^2 \right]
 \le 2\bE \left [ \left  \langle  f_\lambda^\sharp - f_{D, \lambda}^\sharp,\
 \   f_\lambda^\sharp -f_{D, \lambda} \right \rangle _K\right] .$$
Let $$f_{\bx, \lambda}^\sharp  = \bE\left[f_{D, \lambda}^\sharp|D_x\right]
= \Big(\lambda I +\frac 1 N S^*S\Big)^{-2} \Big (2\lambda I + \frac 1 N S^*S\Big )\frac 1 N S^*S f^*$$
and $$f_{\bx, \lambda}  = \bE\left[f_{D, \lambda}|D_x\right]
= \Big(\lambda I +\frac 1 N S^*S\Big )^{-1} \frac 1 N S^*S f^*.$$
Then we have that
\begin{align*}
&  \bE\left[  \left\langle  f_{D, \lambda}^\sharp, \  f_{D, \lambda} \right\rangle_K \right]
-  \bE\left[  \left\langle  f_{\bx, \lambda}^\sharp, \  f_{\bx, \lambda} \right\rangle_K \right]  \\
  = \ & \bE \left[ \left\langle \Big (\lambda I +\frac 1 N S^*S\Big)^{-2} \Big (2\lambda I + \frac 1 N S^*S\Big )\frac 1 N S^*(\by -S f^*), \
 \Big( \lambda I +\frac 1 N S^*S\Big)^{-1} \frac 1 N S^*(\by - Sf^*) \right\rangle_K \right]   \\
=\  & \bE \left[ \left\|   \Big( \lambda I +\frac 1 N S^*S\Big)^{-1} \frac 1 N S^*(\by - Sf^*) \right\|_K^2 \right] +
\bE\left[ \left\| \lambda^{\frac 12} \Big( \lambda I +\frac 1 N S^*S\Big)^{-\frac 32} \frac 1 N S^*(\by - Sf^*) \right\|_K^2\right]  \\
\le\ & 2\lambda^{-2} \bE\left[   \left\|  \frac 1 N S^*(\by - Sf^*) \right\|_K^2 \right]  \le \frac{2\kappa^2\sigma^2 }{\lambda^2 N }.
 \end{align*}
 Therefore,
 	$$
 \bE\left[ \|f_\lambda^\sharp -f_{D, \lambda} \|_K^2 \right] -
 \bE \left[ \|f_{D, \lambda}^\sharp - f_{D, \lambda}\|_K^2 \right]
 \le 2\bE \left [ \left  \langle  f_\lambda^\sharp - f_{x, \lambda}^\sharp,\
 \    f_\lambda^\sharp -f_{x, \lambda} \right\rangle_K\right]  + \frac{4\kappa^2\sigma^2}{\lambda^2 N }.$$
By the fact $\lambda^2 f_{\lambda}^\sharp = (2\lambda L_K + L_K^2 ) (f^*-f_{ \lambda}^\sharp)$ we have
\begin{align*}
 f_\lambda^\sharp - f_{\bx, \lambda}^\sharp & =
\Big (\lambda I +\frac 1 N S^*S\Big)^{-2}\left\{  \Big (2\lambda I + \frac 1 N S^*S\Big ) \frac 1 N S^*S - (2\lambda I + L_K )L_K \right\} (f_{\lambda}^\sharp - f^*) \\
& = \Big (\lambda I +\frac 1 N S^*S\Big)^{-2}
 \Big (2\lambda I + \frac 1 N S^*S\Big ) \Big( \frac 1 N S^*S - L_K \Big) (f_{\lambda}^\sharp - f^*) \\
 & \quad  + \Big (\lambda I +\frac 1 N S^*S\Big)^{-2}  \Big( \frac 1 N S^*S -L_K\Big ) L_K (f_{\lambda}^\sharp - f^*) .
\end{align*}
Applying Lemma \ref{lem:var}  to the random variable  $\xi =  ( f^*(x) - f_\lambda^\sharp(x) ) K_{x} $ and by Lemma \ref{lem:approx2} we obtain
$$ \bE\left[\left \|  \Big( \frac 1 N S^*S - L_K \Big) (f_{\lambda}^\sharp - f^*)
\right \|_K^2\right]  \le   \frac{\kappa^2 \|f_{\lambda}^\sharp - f^*\|_\L2p^2}
{N} \le
\frac {\kappa^2\|u^*\|_\L2p^2 \lambda ^{\min(2r,4)} } {N}.$$
Similarly, we have
$$ \bE\left[\left \|  \Big( \frac 1 N S^*S - L_K \Big) L_K(f_{\lambda}^\sharp - f^*)
\right \|_K^2\right]  \le
\frac {\kappa^2 \| L_K(f_{\lambda}^\sharp - f^*) \| _\L2p^2} {N}
 \le \frac {\kappa^2\|u^*\|_\L2p^2 \lambda ^{\min(2r+2,4) } } {N}.$$
Hence we have
\begin{align*}
\bE\left[ \left\|  f_\lambda^\sharp - f_{\bx, \lambda}^ \sharp \right\|_K^2\right]
\le  \frac{10\kappa^2\|u^*\|_\L2p^2 \lambda ^{\min(2r-2, 0)} }{N}
\end{align*}
and
\begin{align*}
\bE\left[ \left\| \Big(\lambda I + \frac 1N S^*S\Big ) \Big(  f_\lambda^\sharp - f_{\bx, \lambda}^ \sharp  \Big) \right\|_K^2\right]
\le  \frac{10\kappa^2 \|u^*\|_\L2p^2 \lambda ^{\min(2r, 2)} }{N}
\end{align*}
By Lemma \ref{lem:op} and Lemma \ref{lem:S} we have
\begin{align*}
\bE \left[ \left\|\Big(\lambda I + \frac 1N S^*S\Big )^{-1} \Big( f_{\bx,\lambda}^\sharp - f_{\bx, \lambda} \Big) \right\|_K^2\right]
& = \bE\left[ \left\| \lambda\Big(\lambda I + \frac 1N S^*S\Big)^{-3} \frac 1 N S^*S L_K^r u^*\right \|_K^2 \right] \\
&  \le 2\|u^*\|_\L2p^2 \left(  \lambda^{-2} \bE\left[ \left\|L_K^{r-\frac 12} - \Big( \frac 1N S^*S\Big)^{r-\frac 12} \right\|^2 \right] + \lambda^{2r-3} \right)  \\
& \le  2\|u^*\|_\L2p^2 \left(\frac{\kappa^{4r-2}}{\lambda^2 N^{\min(r-\frac 12, 1)}} + \lambda^{\min(2r-3, 0)} \right).
\end{align*}
So we have
\begin{align*}
&\bE \left [ \left  \langle  f_\lambda^\sharp - f_{\bx, \lambda}^\sharp,\
\    f_\lambda^\sharp -f_{\bx, \lambda} \right\rangle_K\right]
 \le \bE \left[  \left \|  f_\lambda^\sharp - f_{\bx, \lambda}^\sharp \right\|_K^2 \right]
+ \bE \left [ \left  \langle  f_\lambda^\sharp - f_{\bx, \lambda}^\sharp,\
\    f_{\bx,\lambda}^\sharp -f_{\bx, \lambda} \right\rangle_K\right] \\
\le \ &  10(\kappa^3+\kappa)\|u^*\|_\L2p^2  \left(  \frac{\lambda ^{\min(2r-2, 0)} }{N}
+ \frac{\lambda^{\min(r-1, 0)} }{N^{\min(\frac r2+\frac 14, 1)}}
+ \frac{ \lambda^{\min(2r-\frac 32, r-\frac 12, 1)} }{N^{\frac 12}}\right)
\end{align*}
and therefore
\begin{align*}
& \lambda\left(  \bE\left[ \|f_\lambda^\sharp -f_{D, \lambda} \|_K^2 \right] -
\bE \left[ \|f_{D, \lambda}^\sharp - f_{D, \lambda}\|_K^2 \right] \right)  \\
\le \ & 20(\kappa^3+\kappa)\|u^*\|_\L2p^2  \left(  \frac{\lambda ^{\min(2r-1, 1)} }{N}
+ \frac{\lambda^{\min(r, 1)} }{N^{\min(\frac r2+\frac 14, 1)}}
+ \frac{ \lambda^{\min(2r-\frac 12, r+\frac 12, 2)} }{N^{\frac 12}}\right)
+ \frac{4\kappa^2\sigma^2}{\lambda N}.
\end{align*}
This together with Lemma \ref{lem:approx2} gives the desired conclusion.
\end{proof}

\section{Error bounds for KRR and distributed KRR}
\label{sec:krr}

The main result of this section is the following error bound  for  KRR
on a single data set.

\begin{theorem}
	\label{thm:krr}
Assume $f^*=L_K^r u^*$ for some $0<r\le 2$ and $u^*\in \L2p.$
\begin{enumerate}
	\item [{\rm (i)}] If $0<r\le \frac 12$, then
	$$\bE \left[ \|f_{D, \lambda}-f^\ast\|_\L2p^2\right]
	\leq  2\lambda^{2r} \|u^*\|_\L2p^2 \left(1+ \frac{\kappa^4}{N^2\lambda^2}+\frac{2\kappa^2}{N\lambda}\right)
	+\left(\frac{\kappa^4}{N^2\lambda^2}+\frac{2\kappa^2}{N\lambda}\right)
	\sigma^2.$$
	\item[{\rm (ii)}]  If $r \ge \frac 12$ and $\lambda$ is chosen so that
	$\lambda N \ge 1$, then there is a constant $c_1>0$ such that
		$$\bE \left[ \|f_{D, \lambda}-f^\ast\|_\L2p^2\right]
		\le c_1 \left( \lambda^{\min(2r,2)}
		+\frac{\lambda^{\min(r,1)}}{N^{\min( \frac r2 +\frac 14, 1)}}
		 + \frac {\sigma^2} {\lambda N} \right).$$
\end{enumerate}
\end{theorem}

The proof of Theorem \ref{thm:krr} will be proved in Section \ref{subsec:krr} below.  With this theorem, the following corollary on the capacity independent rate of KRR follows immediately.

\begin{corollary}\label{cor:krr}
	Under the assumption of Theorem \ref{thm:krr}, by choosing $\lambda = N^{-\max(\frac 1 {1+2r}, \frac 1 3)}$ we have
	$$ \bE \left[ \|f_{D, \lambda}-f^\ast\|_\L2p^2\right]  = O\left (N^{-\min(\frac {2r}{1+2r}, \frac 23)} \right ).$$
	If, in addition, $\sigma^2=0,$ then  we can  choose
	$\lambda = N^{-1}$ to obtain
	$$ \bE \left[ \|f_{D, \lambda}-f^\ast\|_\L2p^2\right]  =
	\begin{cases}
		O\left (N^{-2r} \right ),  &  \hbox{ if } 0<r \le \frac 12 ; \\[0.5em]
		O\left (N^{-(\frac {3r}{2} +\frac 14)} \right ),  &  \hbox{ if } \frac 12 <r \le 1; \\[0.5em]
			O\left (N^{-(\frac {r}{2} +\frac 54)} \right ),  &  \hbox{ if } 1 <r \le \frac 32; \\[0.5em]
			O\left (N^{-2} \right ),  &  \hbox{ if }  r\ge \frac 32 .
	\end{cases} $$
\end{corollary}

While the capacity independent rate of $O(N^{-\frac {2r}{1+2r}})$ is well known in the literature and it is expected the noise free learning should
intuitively be better, it is still surprising to see the rate for noise free learning in Corollary \ref{cor:krr} can be faster than $O(N^{-1})$.
Moreover, KRR was known to suffer from a saturation effect that says  the rate ceases to increase for $r\ge 1.$ We see from Corollary \ref{cor:krr} that for noise free learning the rate can continue to increase beyond $r\ge 1$ and
the saturation effect occurs when $r\ge \frac 3 2.$
To our best knowledge,  both the super fast rate and the relaxed saturation effect for noise free regression learning are novel observations.

We can now prove Theorem \ref{DKRRtheorem} by combining
Theorem \ref{thm:framework} and Theorem \ref{thm:krr}.

\begin{proof}[Proof of Theorem \ref{DKRRtheorem}.]
Theorem \ref{thm:krr} tells that, if $0<r\le \frac 12$, the error bound for KRR with a
single data set of sample size $n$ is
$$\varepsilon(n, \lambda, \sigma^2)  \le
2\lambda^{2r} \|u^*\|_\L2p^2 \left(1+ \frac{\kappa^4}{n^2\lambda^2}+\frac{2\kappa^2}{n\lambda}\right)
+\left(\frac{\kappa^4}{n^2\lambda^2}+\frac{2\kappa^2}{n\lambda}\right)
\sigma^2.$$
Consequently
$$ \varepsilon(n, \lambda, 0)  \le
2\lambda^{2r} \|u^*\|_\L2p^2 \left(1+ \frac{\kappa^4}{n^2\lambda^2}+\frac{2\kappa^2}{n\lambda}\right).
$$
By Theorem \ref{thm:framework} and the fact $N=mn$, we obtain
\begin{align*}
 \bE \left[ \|\overline { f_{D, \lambda}}-f^\ast\|_\L2p^2\right] & \le
 \frac{\varepsilon(n, \lambda, \sigma^2)}{m} + \varepsilon(n, \lambda, 0) \\
 & \le 4\lambda^{2r} \|u^*\|_\L2p^2 \left(1+ \frac{\kappa^4}{n^2\lambda^2}+\frac{2\kappa^2}{n\lambda}\right)
 +\left(\frac{\kappa^4}{mn^2\lambda^2}+\frac{2\kappa^2}{mn\lambda}\right)
 \sigma^2 \\
& \le C_1 \left\{ \lambda^{2r} \left( 1 + \frac{m^2}{N^2\lambda^2} + \frac{m} {N\lambda} \right) + \frac {\sigma^2 m}{N^2} + \frac {\sigma^2} {N\lambda}\right\}
\end{align*}
with $C_1=(1+\kappa^4)\max(4\|u^*\|_\L2p^2, 1).$
If $\lambda = N^{-\frac 1{2r+1}}$ and $m\le N^{\frac{2r}{1+2r}}$, then
$m\le N\lambda$ and therefore
\begin{align*}
\bE \left[ \|\overline {f_{D, \lambda}}-f^\ast\|_\L2p^2\right]  \le
 3C_1 \left\{ \lambda^{2r}  + \frac {\sigma^2} {N\lambda}\right\}
 = 3 C_1(1+\sigma^2) N^{-\frac {2r}{1+2r}}.
\end{align*}
This proves the claim in (i).

To prove claim (ii) we note that if $\frac 12 \le r \le 1$ and $\lambda n\ge 1$, the error bound
for KRR is
$$ \varepsilon(n, \lambda, \sigma^2) \le c_1 \left(  \lambda^{2r}
+\frac{\lambda^{r}}{n^{ \frac r2 +\frac 14} }
+ \frac {\sigma^2} {\lambda n} \right). $$
By Theorem \ref{thm:framework} we have
$$ \bE \left[ \|\overline {f_{D, \lambda}}-f^\ast\|_\L2p^2\right]
\le 2c_1 \left(  \lambda^{2r}
+\frac{\lambda^{r}}{n^{ \frac r2 +\frac 14} }
+ \frac {\sigma^2} {\lambda n m}\right)
= 2c_1 \left(  \lambda^{2r}
+\frac{\lambda^{r} m^{ \frac r2 +\frac 14}} {N^{ \frac r2 +\frac 14} }
+ \frac {\sigma^2} {\lambda N} \right).
$$
With the choice $\lambda = N^{-\frac 1{2r+1}}$ and $m\le N^{\frac{4r^2+1}{4r^2+4r+1}}$, we obtain
$$ \bE \left[ \|\overline {f_{D, \lambda}}-f^\ast\|_\L2p^2\right]
\le 2c_1(2+\sigma^2) N^{-\frac {2r}{2r+1}}.$$
This finishes the proof.
\end{proof}

\subsection{Leave one out analysis for KRR}
\label{subsec:krr}

In this subsection we bound the error of KRR on a single data set
and prove Theorem \ref{thm:krr}. To this end, we perform a leave one out analysis of KRR. While the idea of leave one out analysis roots in the work \cite{zhang2003leave}, we need a more rigorous analysis to quantify the dependence on $\sigma^2.$

Let $\widetilde{D}$ be a sample of $N+1$ observations $\tilde D = \{z_i=(x_i,\, y_i): 1\leq i \leq N+1\}$, and $\widetilde{D}_{(i)}$ denote the sample of $N$ observations generated by removing the $i$th observation $(x_i, y_i)$ from $\widetilde{D}$, that is, $\widetilde{D}_{(i)} = \{z_j=(x_j, y_j): 1\le j\le N+1, j\not= i\}.$ Let
\begin{equation*}
\hat g =
\arg\min_{f\in \H_K} \left \{ \frac{1}{N} \sum_{j=1}^{N+1}
(f(x_j)-y_j)^2+\lambda \|f\|_K ^2  \right \}
\end{equation*}
and define the leave one out estimators associated to $\widetilde{D}_{(i)}$ by
\begin{equation*}
\hat f_{(i)} =
\arg\min_{f\in \H_K} \left \{ \frac{1}{N} \sum_{{j=1}\atop{j\neq i}}^{N+1}
(f(x_j)-y_j)^2+\lambda \|f\|_K ^2  \right \} .
\end{equation*}
It is easy to verify that
$\hat g = f_{\widetilde{D}, \frac{N\lambda}{N+1}}$ and
$f_{D, \lambda}=\hat f_{(N+1)}.$ We also note that
$\hat f_{(i)}$, $1\le i\le N+1$ are identically distributed (but not independent).

The following bound for the leave one out estimator is referred to leave one out stability and has been proved in \cite{zhang2003leave}. We still give the proof for the purpose of self-containing.
\begin{lemma}\label{zhangleaveoneout}
	We have
	\begin{equation*}
	\|\hat f_{(i)} - \hat g\|_K  \leq \frac{\kappa}{N\lambda} |\hat g(x_i)-y_i|.
	\end{equation*}
\end{lemma}

\begin{proof}
We denote $ \triangle f_i= \hat f_{(i)}-\hat g$ and
$$ \calE_{i}(f) = \sum_{{j=1}\atop {j\neq i}}^{N+1} \big(f(x_j)-y_j
\big)^2. $$ Recall that, if $h$ is convex, then for any
$a, b\in \mathbb R$ and $ t \in [0,1]$, there holds
$$ h(a+t(b-a))-h(a) \leq t\big(h(b)-h(a)\big).$$
By the convexity of the square loss, we have
\begin{align}
\calE_{i}(\hat g+t\triangle f_i)
-\calE_{i}(\hat g) & \leq
t\left(\calE_{i}(\hat f_{(i)})-\calE_{i}(\hat g) \right) \label{eq:loo-1}
\end{align}
and
\begin{align}
\calE_{i}(\hat f_{(i)}-t\triangle f_i) -\calE_{i}(\hat f_{(i)}) &
\leq t\left( \calE_{i}(\hat g)-\calE_{i}(\hat f_{(i)}) \right) . \label{eq:loo-2} \end{align}
By the definition of $\hat f_{(i)}$, we have
\begin{equation*}
\frac{1}{N}\calE_{i}(\hat f_{(i)})+\lambda \left \|\hat f_{(i)} \right \|_K^2
 \leq \frac{1}{N}\calE_{i}(\hat f_{(i)}-t\triangle f_i)+\lambda\left \|\hat f_{(i)}-t\triangle f_i \right \|_K^2.
\end{equation*}
This together with \eqref{eq:loo-2} leads to
\begin{equation}
2t \left \langle \hat f_{(i)},\ \triangle f_i\right\rangle_K - t^2 \| \triangle f_i\|_K^2 =
\left \|\hat f_{(i)} \right \|_K^2 -\left \|\hat f_{(i)}-t\triangle f_i \right \|_K^2
\le \frac t {N\lambda}  \left(\calE_{i}(\hat g)-  \calE_{i}(\hat f_{(i)}) \right) .
\label{eq:df1}
\end{equation}
Similarly, by the definition of $\hat g$, we have
\begin{align*}
 \frac{1}{N}\calE_{i}(\hat g)+\lambda
\left \|\hat g \right \|_K^2 + \frac{1}{N}\left (\hat g(x_i)-y_i \right )^2 &  \leq
\frac{1}{N}\calE_{i}(\hat g+t\triangle
f_i)+\lambda \left \|\hat g+t\triangle f_i \right \|_K^2 \\[1ex]
& \qquad +\frac{1}{N}
\left( \hat g(x_i) + t\triangle f_i(x_i) -y_i\right)^2.
\end{align*}
By \eqref{eq:loo-1}, we obtain
\begin{align}
& -2t \left \langle \hat g,\ \triangle f_i\right\rangle_K - t^2 \| \triangle f_i\|_K^2
  = \Big \|\hat g \Big\|_K^2 -  \Big \|\hat g+t\triangle f_i \Big\|_K^2  \nonumber  \\
\leq\ & \frac{t}{N \lambda }
\Big( \calE_{i}(\hat  f_{(i)})-\calE_{i}(\hat g) \Big)
+\frac{1}{N\lambda }
\left\{ \left( \hat g(x_i) + t\triangle f_i(x_i) -y_i\right)^2 - \left(\hat g(x_i)-y_i \right )^2 \right\} .
\label{eq:df2}
\end{align}
Adding \eqref{eq:df1} and \eqref{eq:df2} together, we obtain
\begin{align*}
2t(1-t) \left \|\triangle f_i \right\|_K^2   & \le
\frac{1}{N\lambda }
\left\{ \left( \hat g(x_i) + t\triangle f_i(x_i) -y_i\right)^2 - \left(\hat g(x_i)-y_i \right )^2 \right\}  \\[0.5em]
& = \frac {t}{N\lambda }  \triangle f_i(x_i) \Big( 2 \hat g(x_i) +t \triangle f_i(x_i)
-2 y_i \Big) \\[0.5em]
& \le \frac {\kappa t}{N\lambda }  \| \triangle f_i \|_K  \Big | 2 \hat g(x_i) +t \triangle f_i(x_i) -2 y_i \Big |.
\end{align*}
Therefore,
\begin{equation*}
2(1-t) \left \|\triangle f_i \right\|_K   \le \frac {\kappa }{N\lambda }  \Big | 2 \hat g(x_i) +t \triangle f_i(x_i) -2 y_i \Big |.
\end{equation*}
Letting $ t\to 0 $, we obtain
\begin{align*}
\left \|\triangle f_i \right\|_K   \le \frac {\kappa }{N\lambda }  \Big |  \hat g(x_i)  -y_i \Big |.
\end{align*}
This proves  Lemma \ref{zhangleaveoneout}.
\end{proof}

We will need the following lemma.
\begin{lemma}
	\label{lemma:Q}
 We have
$$ \bE\left[ \frac{1}{N+1}\sum_{i=1}^{N+1}\big(\hat g(x_i)-y_i\big)^2 \right]
\le Q\Big(\sfrac{N\lambda }{N+1}, N+1\Big) + \sigma^2 .$$
\end{lemma}

\begin{proof} Let $\tilde\lambda = \frac{N\lambda}{N+1}$ and $f_{\tilde\lambda}$ be defined by \eqref{flambda} associated with
	the regularization parameter $\tilde{\lambda}.$ By the definition
	of $\hat g$, we have
 \begin{align*}
 \bE \left[ \frac{1}{N+1} \sum_{i=1}^{N+1}\big(\hat g(x_i)-y_i\big)^2 \right]
 &= \frac{N}{N+1} \bE \bigg[\frac{1}{N} \sum_{i=1}^{N+1}\big(\hat g(x_i)-y_i\big)^2+\lambda\left [ \|\hat g\|_K^2 \right]
 \bigg]-\frac{N\lambda }{N+1} \bE \|\hat g\|_K^2  \\[0.5em]
 &\leq \frac{N}{N+1}\bE\bigg[\frac{1}{N}\sum_{i=1}^{N+1}
 \big(f_{\tilde \lambda }(x_i)-y_i\big)^2+\lambda\| f_{\tilde \lambda} \|_K^2
 \bigg]-\frac{N\lambda }{N+1} \bE \left[ \|\hat g\|_K^2 \right] \\[0.5em]
 &=\|f_{\tilde \lambda }-f^\ast\|_\L2p^2+\tilde\lambda
 \|f_{\tilde \lambda }\|_K^2-\tilde \lambda  \bE\|\hat g\|_K^2+\sigma^2\\[0.5em]
 & = Q\Big(\tilde\lambda , N+1 \Big) + \sigma^2,
 \end{align*}
where we have used the fact that $f_{\tilde{\lambda}}$ is independent of $(x_i, y_i)$ and the identity  \eqref{eq:fdec}.
\end{proof}

\begin{proposition} \label{krrerrorbound}
There holds that
\begin{equation*}
\bE \left[ \|f_{D, \lambda}-f^\ast\|_\L2p^2\right]
\leq  \left(1+ \frac{\kappa^4}{N^2\lambda^2}+\frac{2\kappa^2}{N\lambda}\right)
Q\Big(\sfrac{N\lambda }{N+1}, N+1 \Big)
+\left(\frac{\kappa^4}{N^2\lambda^2}+\frac{2\kappa^2}{N\lambda}\right)
\sigma^2.
\end{equation*}
\end{proposition}

\begin{proof}
By the facts that$ f_{D,\lambda }=\hat f_{(N+1)}$ and
	$\hat f_{(i)}$, $1\le i\le N+1$ are identically distributed, we have
\begin{equation*}
	\bE \left[ \|f_{D, \lambda}-f^\ast\|_\L2p^2\right]
	=	\bE \left[ \|\hat f_{(N+1)}-f^\ast\|_\L2p^2\right]
	= \frac 1 {N+1} \sum_{i=1}^{N+1}
		\bE \left[ \|\hat f_{(i)}-f^\ast\|_\L2p^2\right]
\end{equation*}
By the fact that $f_{(i)}$ is independent of $(x_i, y_i)$ and the identity \eqref{eq:fdec} we obtain
\begin{align*}
	\bE \left[ \|f_{D, \lambda}-f^\ast\|_\L2p^2\right]
	& = \frac 1 {N+1} \sum_{i=1}^{N+1} \bE \left[ \left ( \hat f_{(i)}(x_i) - f^*(x_i) \right)^2 \right ] \\
& = \frac 1 {N+1} \sum_{i=1}^{N+1} \Big( \bE\left[ (\hat f_{(i)}(x_i) -y_i )^2 \right] - \sigma^2\Big) \\
&=\bE \left[ \frac{1}{N+1}\sum_{i=1}^{N+1}\big(\hat f_{(i)}(x_i)-y_i\big)^2 \right] -\sigma^2
\nonumber \\
&=\bE \left [\frac{1}{N+1}\sum_{i=1}^{N+1}\big(\hat f_{(i)}(x_i)-y_i\big)^2-\frac{1}{N+1}\sum_{i=1}^{N+1}\big(\hat g(x_i)-y_i\big)^2\right]\\
& \qquad  + \left( \bE\left[ \frac{1}{N+1}\sum_{i=1}^{N+1}\big(\hat g(x_i)-y_i\big)^2 \right] -\sigma^2\right)
\end{align*}
By Lemma \ref{zhangleaveoneout} and the reproducing property we have
$$|\hat f_{(i)}(x_i) - \hat g(x_i)| \le \kappa \| \hat f_{(i)} - \hat g \|_K \le
\frac{\kappa^2}{N\lambda} |\hat g(x_i) - y_i|.$$
Therefore,
\begin{align}
& \bE \left [\frac{1}{N+1}\sum_{i=1}^{N+1}\big(\hat f_{(i)}(x_i)-y_i\big)^2-\frac{1}{N+1}\sum_{i=1}^{N+1}\big(\hat g(x_i)-y_i\big)^2\right] \nonumber \\
= \ & \bE \left[ \frac{1}{N+1} \sum_{i=1}^{N+1}  \big(\hat f_{(i)}(x_i)-\hat g(x_i) \big) \big(\hat f_{(i)}(x_i)+\hat g(x_i)-2y_i \big) \right ]  \nonumber \\
= \ & \bE \left [ \frac{1}{N+1} \sum_{i=1}^{N+1} \big(\hat f_{(i)}(x_i)-\hat g(x_i) \big)^2+ \frac{2}{N+1} \sum_{i=1}^{N+1} \big(\hat f_{(i)}(x_i)-\hat g(x_i) \big)
\big(\hat g(x_i)-y_i\big) \right ]  \nonumber \\
\leq\ &  \bE \left[ \frac{1}{N+1} \sum_{i=1}^{N+1}
\frac{\kappa^4}{N^2\lambda^2} \big(\hat g(x_i)-y_i
\big)^2+ \frac{2}{N+1} \sum_{i=1}^{N+1}
\frac{\kappa^2}{N\lambda} \big(\hat g(x_i)-y_i
\big)^2 \right]  \nonumber \\
=\ & \bigg(\frac{\kappa^4}{N^2\lambda^2}+\frac{2\kappa^2}{N\lambda}
\bigg) \bE \left[ \frac{1}{N+1}\sum_{i=1}^{N+1}\big(\hat g(x_i)-y_i\big)^2 \right].
\nonumber
\end{align}
Then
$$	\bE \left[ \|f_{D, \lambda}-f^\ast\|_\L2p^2\right]
\le  \bigg(1+ \frac{\kappa^4}{n^2\lambda^2}+\frac{2\kappa^2}{n\lambda}
\bigg) \bE \left[ \frac{1}{n+1}\sum_{i=1}^{n+1}\big(\hat g(x_i)-y_i\big)^2\right]
 -\sigma^2$$
 and the desired conclusion follows by using Lemma \ref{lemma:Q}.
\end{proof}

Note that Theorem \ref{thm:krr} follows immediately from Proposition \ref{krrerrorbound} and Lemma \ref{lem:Qbound}.

\section{Error bound for BCKRR and distributed BCKRR}
\label{sec:bckrr}

  The main result of this section is the following  error bounds
  for BCKRR on a single data set that will be used in combination with
  Theorem \ref{thm:framework} to derive the error bounds for
  its distributed version.

\begin{theorem}
	\label{thm:loo}
	Assume $\sigma^2<\infty$ and $f^*=L_K^r u^*$ for some $0<r\le 2$ and $u^*\in \L2p.$
	\begin{enumerate}
		\item [{\rm (i)}] If $0<r\le \frac 12$, then
		$$\bE \left[ \|f_{D, \lambda}^\sharp-f^\ast\|_\L2p^2\right]
		\leq  2\lambda^{2r} \|u^*\|_\L2p^2 \left(1+ \frac{\kappa^4}{N^2\lambda^2}+\frac{2\kappa^2}{N\lambda}\right)
		+4\sigma^2 \left(\frac{\kappa^4}{N^2\lambda^2}+\frac{\kappa^2}{N\lambda}\right)
		.$$
		\item[{\rm (ii)}]  If $\frac 12 <r\le 1$ and
		$\lambda N \ge 1$, then there is a constant $\tilde c_1>0$ such that
		$$\bE \left[ \|f_{D, \lambda}^\sharp-f^\ast\|_\L2p^2\right]
		\le \tilde c_1 \left( \lambda^{2r} +\frac{\lambda^{r}}{N^{\frac r2 +\frac 14}}   + \frac {\sigma^2} {\lambda N} \right).$$
		
			\item[{\rm (iii)}]  If $ 1 <r\le \frac 32 ,$ $\lambda\le 1 $ and  $\lambda N \ge 1$, then there is a constant $\tilde c_2>0$ such that
		$$\bE \left[ \|f_{D, \lambda}^\sharp-f^\ast\|_\L2p^2\right]
		\le \tilde c_2 \left( \lambda^{2r} +\frac{\lambda}{N^{\frac r2 +\frac 14}}  + \frac{\lambda^{r +\frac 12}}{\sqrt{N}} + \frac {\sigma^2} {\lambda N} \right).$$
		
			\item[{\rm (iv)}]  If $ r\ge \frac 32 ,$ $\lambda\le 1 $ and  $\lambda N \ge 1$, then there is a constant $\tilde c_3>0$ such that
		$$\bE \left[ \|f_{D, \lambda}^\sharp-f^\ast\|_\L2p^2\right]
		\le \tilde c_3 \left( \lambda^{2r}  + \frac{\lambda}{N} + \frac {\sigma^2} {\lambda N} +\frac{\lambda^2}{\sqrt{N}} \right).$$
		
	\end{enumerate}
\end{theorem}

The proof of Theorem \ref{thm:loo} will be proved in Section \ref{subsec:loo-bckrr} below. The following corollary on the capacity independent rate of BCKRR follows immediately.

\begin{corollary}
	Under the assumption of Theorem \ref{thm:loo}, by choosing $\lambda = N^{-\frac 1 {1+2r}}$ we have
	$$ \bE \left[ \|f_{D, \lambda}^\sharp-f^\ast\|_\L2p^2\right]  = O\left (N^{-\frac {2r}{1+2r}} \right ).$$
	If, in addition, $\sigma^2=0,$
	 then  we can  choose
	$\lambda = N^{-1}$ to obtain
	$$ \bE \left[ \|f_{D, \lambda}^\sharp-f^\ast\|_\L2p^2\right]  =
	\begin{cases}
	O\left (N^{-2r} \right ),  &  \hbox{ if } 0<r \le \frac 12 ; \\[0.5em]
	O\left (N^{-(\frac {3r}{2} +\frac 14)} \right ),  &  \hbox{ if } \frac 12 <r \le 1; \\[0.5em]
	O\left (N^{-(\frac {r}{2} +\frac 54)} \right ),  &  \hbox{ if } 1 <r \le \frac 32; \\[0.5em]
	O\left (N^{-2} \right ),  &  \hbox{ if }  r\ge \frac 32 .
	\end{cases} $$
\end{corollary}

BCKRR suffers from a saturation effect when $r\ge 2$, which relaxed the saturation effect of KRR  \cite{guo2017learning}. For noise free learning, however, we see the result for BCKRR is the same as that for KRR in Corollary \ref{cor:krr} and the saturation occurs when $r= \frac 32$. A plausible interpretation of this phenomenon is that the rate $O(N^{-2})$ might be  the fastest rate a regression learning algorithm can achieve and thus bias correction cannot help further improve the rate.

Theorem \ref{thm:2} can now be proved by combining
Theorem \ref{thm:framework} and Theorem \ref{thm:loo}.
The proof is similar to that of Theorem \ref{DKRRtheorem}
in Section \ref{sec:krr}. We omit the details.

\subsection{Two alternative perspectives for bias correction}
\label{sec:perspects}

Recall the BCKRR estimator is defined by a two step procedure. Although it admits
an explicit expression and an operator representation, neither is suitable for leave one out analysis, if not impossible.
In this section  we show that BCKRR can be interpreted by two alternative perspectives,  fitting the residual and recentering regularization. The latter represents BCKRR as a Tikhonov regularization scheme and plays an essential role in the leave one out analysis of BCKRR in the next subsection.  These two perspectives on bias correction
may also be of independent interest by themselves; see the discussions Section \ref{sec:conclusion}.

Suppose KRR estimator $f_{D,\lambda}$ fit the data by trading off the fitting error and model complexity.
If we further fit the residuals by a function $\hat h$ and add it to $f_{D,\lambda}$,
the resulted function will have smaller fitting error. In other words,  the bias is reduced.  The following proposition tells that
BCKRR defined in \eqref{eq:fa} is equivalent to this process.

\begin{proposition}
	\label{prop:fitr}
	Let $r _i=y_i - f_{D,\lambda}(x_i)$ be the residual and
	$\hat h$ is the KRR estimator obtained by fitting the residual, i.e.,
	$$ \hat h = \arg\min_{g\in \H_K} \left\{
	\frac 1N \sum_{i=1}^N (r_i-g(x))^2 + \lambda \|g\|_K^2\right\}.$$
	Then we have $f^\sharp_{D,\lambda} = f_{D,\lambda} + \hat h.$
\end{proposition}

\begin{proof} Denote by $\mathbf r = (r_1, r_2, \ldots, r_n)^\top$ the column vector of
	residuals.  By the operator representation of the KRR estimator we have
	\begin{equation}
	\label{eq:goper}
	\hat h = \frac 1 N \left(\lambda I + \frac 1 N S^*S \right)^{-1} S^* \mathbf r.
	\end{equation}
	It can be easily verified that
	\begin{align*}
	S^* \mathbf r & = S^* (\by  - S f_{D,\lambda})
	=  S^*\by  - \frac 1 N S^*S  \left(\lambda I + \frac 1 N S^*S\right)^{-1} S^* \by  \\
	& =  \lambda \left( \lambda I +\frac 1 N S^*S\right)^{-1} S^*\by  = \lambda N f_{D,\lambda}.
	\end{align*}
	Plugging it into \eqref{eq:goper} we obtain
	$$ \hat h = \lambda \left(\lambda I + \frac 1 N S^*S\right)^{-1} f_{D,\lambda}
	= f^\sharp_{D,\lambda} - f_{D,\lambda}.$$
	This proves the conclusion.
\end{proof}

Recall that KRR is somewhat equivalent to search a minimizer of the fitting error
in a ball of radius $\frac 1 {\sqrt \lambda}$ centered at the zero function.
The following proposition explains BCKRR as a re-search for a
minimizer of the fitting error in a ball centered at $f_{D,\lambda}.$
This is somewhat equivalent to increasing the searching region to reduce the fitting error  and thus implement bias reduction.

\begin{proposition}
	\label{prop:recenter}
	We have
	\begin{equation}
	f^\sharp_{D,\lambda }= \arg\min_{f\in \H_K} \left \{ \frac{1}{N}
	\sum_{i=1}^N (y_i -f(x_i))^2+\lambda \|f-f_{D,\lambda} \|_K ^2  \right\}.
	\label{description}
	\end{equation}
\end{proposition}

\begin{proof} By the fact $y_i = f_{D,\lambda}(x_i) + r_i$ and Proposition \ref{prop:fitr},
	for all $f\in\H_K$, we have
	\begin{align*}
	\frac 1 N \sum_{i=1}^N (y_i -  f^\sharp_{D,\lambda }(x_i))^2 + \lambda\|f^\sharp_{D,\lambda} - f_{D,\lambda}\|_K^2
	& = \frac 1N  \sum_{i=1}^N (r_i- \hat h(x_i) )^2 + \lambda \|\hat h \|_K^2 \\
	& \le \frac 1N \sum_{i=1}^N (r_i - [f(x_i) - f_{D,\lambda}(x_i)])^2 + \lambda \|f- f_{D,\lambda}\|_K^2 \\
	& = \frac 1N \sum_{i=1}^N (y_i - f(x_i))^2 + \lambda \|f-f_{D,\lambda}\|_K^2.
	\end{align*}
	Therefore, $f^\sharp_{D,\lambda}$ is a minimizer of the recentered risk functional
	$\frac 1N \sum_{i=1}^N (y_i - f(x_i))^2 + \lambda \|f-f_{D,\lambda}\|_K^2.$
	Since the recentered risk functional is convex in $f$, the minimizer is unique.
	This proves \eqref{description}.
\end{proof}

Proposition \ref{prop:recenter} writes the BCKRR as a  recentering Tikhonov regularization scheme. It enables the use of convex analysis to derive the
leave one out error bound for BCKRR.

\subsection{Leave one out analysis of BCKRR}
\label{subsec:loo-bckrr}

We now perform a leave one analysis of BCKRR by the similar framework as that we have done in Section \ref{sec:krr}. The result will be used to prove Theorem \ref{thm:loo}.  To this end, recall the definitions of $\widetilde{D}$, $\hat f_{(i)}$ and $\hat g$ in Section \ref{subsec:krr}. We define
\begin{align}
&\hat f_{(i)}^\sharp=\arg\min_{f \in {{\cal H}_K}}
\bigg \{\frac{1}{N}\sum_{{j=1}\atop{j\neq i}}^{N+1} \big(f(x_j)-y_j
\big)^2+\lambda\|f- \hat f_{(i)}\|_K^2 \bigg\} \label{f-isharp} \\
&\hat g^\sharp=\arg\min_{f \in {{\cal H}_K}} \bigg \{\frac{1}{N}\sum_{j=1}^{N+1}
\big(f(x_j)-y_j \big)^2+\lambda \|f-\hat g\|_K^2 \bigg\}. \label{gsharp}
\end{align}
Notice that that $\hat f_{(N+1)}^\sharp = f^\sharp_{D,\lambda}$ and $\hat g^\sharp
= f^\sharp_{\widetilde D, \tilde\lambda}$ where again $\tilde\lambda = \frac{N\lambda}{N+1}.$
\vspace{0.2cm}

\begin{lemma}\label{swleaveoneout}
	For all $1\leq i\leq n+1$, there is
$$\|\hat f_{(i)}^\sharp-\hat g^\sharp\|_K \leq
\frac{\kappa}{N\lambda}\Big[|\hat g (x_i)-y_i|+|\hat g^\sharp(x_i)-y_i|\Big].$$
\end{lemma}

\begin{proof} Recall
 $$\calE_{i}(f) = \sum_{{j=1}\atop {j\neq i}}^{N+1} \big(f(x_j)-y_j
\big)^2. $$ Let  $ \triangle f_i = \hat f^\sharp_{(i)} - \hat g^\sharp.$ By the convexity of the square function,
for any $t \in (0,1)$,  we have
\begin{equation}
\calE_{i} (\hat g^\sharp + t\triangle f_i) - \calE_{i} (\hat g^\sharp) \leq
t\big[\calE_{i} (\hat f^\sharp_{(i)}) - \calE_{i} (\hat g^\sharp) \big],
\label{eq:Ei1}
\end{equation}
and
\begin{equation}
\calE_{i} (\hat f^\sharp_{(i)} - t \triangle f_i) - \calE_{i}(\hat f^\sharp_{(i)}) \leq
t \big[\calE_{i} (\hat g^\sharp)-\calE_{i} (\hat f^\sharp_{(i)} ) \big].
\label{eq:Ei2}
\end{equation}
By the definition of $\hat f^\sharp_{(i)} $ in \eqref{f-isharp}, for any $t\in (0, 1)$
\begin{align*}
\frac{1}{N} \calE_{i} (\hat f^\sharp_{(i)}) + \lambda \|\hat f^\sharp_{(i)} -\hat f_{(i)}\|_K^2
 \leq \frac{1}{N} \calE_{i} (\hat f^\sharp_{(i)}-t\triangle f_i)
  +\lambda \|\hat f^\sharp_{(i)} -t\triangle f_i - \hat f_{(i)}\|_K^2.
\end{align*}
This yields
\begin{align}
\lambda\left\{ 2t \left\langle \hat f^\sharp_{(i)}-\hat f_{(i)},\, \triangle f_i \right \rangle_K
-t^2\|\triangle f_i\|_K^2 \right\}
 \leq \frac 1N \left\{\calE_{i} (\hat f^\sharp_{(i)} - t \triangle f_i) - \calE_{i}(\hat f^\sharp_{(i)})\right\} .
  \label{formular2}
\end{align}

Similarly, by the definition of $\hat g^\sharp$ in \eqref{gsharp}, we have for any $t\in(0,1)$
\begin{align*}
& \frac 1 N \calE_{i}(\hat g^\sharp) + \frac{1}{N}(\hat g^\sharp(x_i)-y_i)^2
+  \lambda \|\hat g^\sharp-\hat g\|_K^2  \\
\le  & \frac{1}{N} \calE_{i}  (\hat  g^\sharp + t\triangle f_i)
+\frac{1}{N} \big((1-t) \hat g^\sharp(x_i)+t \hat f^\sharp_{(i)} (x_i)-y_i\big)^2
+\lambda \|\hat g^\sharp - \hat g + t\triangle f_i \|_K^2.
\end{align*}
Thus,
\begin{align}
&\lambda\left\{ -2t \left\langle \hat g^\sharp - \hat g, \triangle f_i \right\rangle_K -t^2\|\triangle f_i\|_K^2 \right\}
  \leq \frac{1}{N} \left\{ \calE_{i} (\hat g^\sharp +  t \triangle f_i)-\calE_{i} (\hat g^\sharp)\right\} \nonumber \\
& \qquad \qquad +\frac{t}{N} \left(\hat f^\sharp_{(i)} (x_i)-\hat g^\sharp (x_i)\right)
\left ( (2-t) \hat g^\sharp(x_i)+ t \hat f^\sharp_{(i)}(x_i)-2y_i\right). \label{fomular3}
\end{align}
Adding \eqref{formular2} and \eqref{fomular3} together and applying \eqref{eq:Ei1}, \eqref{eq:Ei2}, we obtain
\begin{align*}
2(1-t) \|\triangle f_i\|_K^2 & +2\left\langle \hat g - \hat f_{(i)},\, \triangle f_i\right\rangle_K
\le  \frac{1}{N\lambda}\left( \hat f^\sharp_{(i)} (x_i)-\hat g^\sharp (x_i)\right)
\left((2-t) \hat g^\sharp(x_i) + t \hat f^\sharp_{(i)}(x_i)-2y_i\right).
\end{align*}
Letting $ t \to 0,$ we have
\begin{align*}
\|\triangle f_i\|_K^2 + \left\langle \hat g - \hat f_{(i)},\, \triangle f_i \right\rangle_K \leq \frac{1}{N\lambda}\left( \hat f^\sharp_{(i)}(x_i)-\hat g^\sharp(x_i)\right)
\left ( \hat g^\sharp(x_i)-y_i\right)= \frac{1}{N\lambda} \triangle f_i(x_i)
\left ( \hat g^\sharp(x_i)-y_i\right).
 \end{align*}
 Therefore,
\begin{align*}
\|\triangle f_i\|_K^2  \leq  \| \hat g - \hat f_{(i)} \|_K \|\triangle f_i\|_K +
\frac{\kappa}{N\lambda} \|\triangle f_i\|_K \big| \hat g^\sharp (x_i)-y_i\big|,
\end{align*}
which implies
\begin{align*}
\|\triangle f_i\|_K  \leq  \| \hat g - \hat f_{(i)} \|_K  +
\frac{\kappa}{N\lambda} \big| \hat g^\sharp (x_i)-y_i\big|,
\end{align*}
By Lemma \ref{zhangleaveoneout},
we obtain the desired bound.
\end{proof}

\begin{proposition}\label{fhatsharpsample}
	We have
\begin{align*}
\bE \left[ \| f^\sharp_{D,\lambda} -f^*\|_{\L2p}^2\right]
 \leq  \left(1 + \frac{4\kappa^2}{N \lambda} + \frac{4\kappa^4}{N^2 \lambda^2} \right) Q\Big(\sfrac{N \lambda}{N+1} ,\  N+1\Big)
+ \left(\frac{4\kappa^2}{N \lambda} + \frac{4\kappa^4}{N^2 \lambda^2} \right)  \sigma^2.
\end{align*}
\end{proposition}

\begin{proof}
Note that $\hat f^\sharp_{(i)}$, $1\le i\le N+1,$ are identically distributed,
and $\hat f^\sharp_{(i)}$ is independent of the observation $(x_i, y_i).$
By  the fact $f^\sharp_{D,\lambda } = \hat f^\sharp_{(N+1)},$ we have
\begin{align}
\bE \left[  \|f^\sharp_{D, \lambda} -f^*\|_{\L2p}^2 \right] &
=\bE\left[\frac 1 {N+1} \sum_{i=1}^{N+1}  \left\| \hat f^\sharp_{(i)} - f^*\right \|_{\L2p}^2 \right]  \nonumber \\
&= \bE\left[ \frac 1 {N+1} \sum_{i=1}^{N+1} (\hat f^\sharp_{(i)} (x_i) - y_i)^2 \right] - \sigma^2  \nonumber \\
& = \left\{\bE\left[ \frac 1 {N+1} \sum_{i=1}^{N+1} (\hat g^\sharp(x_i)-y_i)^2 \right] - \sigma^2\right\}  \nonumber  \\
& \qquad +
\bE \left[ \frac 1 {N+1} \sum_{i=1}^{N+1}
\left(\hat f ^\sharp _{(i)}(x_i) - \hat g^\sharp (x_i)\right)^2\right ] \nonumber  \\
& \qquad + 2 \bE \left[ \frac 1 {N+1} \sum_{i=1}^{N+1}
\left(\hat f ^\sharp _{(i)}(x_i) - \hat g^\sharp (x_i)\right)
\left( \hat g^\sharp (x_i) -y_i\right)\right] \nonumber \\
&:=J_1 + J_2+J_3. \label{eq:sharpdec}
\end{align}
Recall Lemma \ref{lemma:Q} tells that
\begin{align}
\bE \left[ \frac 1 {N+1} \sum_{i=1}^{N+1} (\hat g(x_i)-y_i)^2 \right]
\leq Q\Big(\sfrac{N \lambda}{N+1} ,\  N+1\Big)+\sigma^2.
\label{eq:gerror}
\end{align}
By the definition of $\hat g^\sharp$ we also obtain
\begin{align}
\bE \left[\frac 1 {N+1} \sum_{i=1}^{N+1} (\hat g^\sharp (x_i)-y_i)^2 \right]
 \le \bE\left[ \frac 1 {N+1} \sum_{i=1}^{N+1} (\hat g (x_i)-y_i)^2  \right]\leq Q\Big(\sfrac{N \lambda}{N+1} ,\  N+1\Big)+\sigma^2.
\label{eq:gaerror}
\end{align}
Therefore,  we can bound $J_1$ by
\begin{equation}
\label{eq:J1}
J_1 \le  Q\Big(\sfrac{N\lambda}{N+1} ,\  N+1\Big).
\end{equation}
To estimate $J_2$, we apply Lemma \ref{swleaveoneout} and obtain
\begin{align}
\left(\hat f ^\sharp _{(i)}(x_i) - \hat g ^\sharp (x_i)\right)^2
& \le  \kappa^2 \| \hat f ^\sharp _{(i)} - \hat g^\sharp \| _K ^2  \le   \kappa^2 \left(\frac{\kappa}{N\lambda}|\hat g(x_i)-y_i|
+\frac{\kappa}{N\lambda}|\hat g^\sharp(x_i)-y_i|\right)^2 \nonumber \\[1ex]
& \le 2\kappa^2 \left(\frac{\kappa^2}{N^2\lambda^2}(\hat g(x_i)-y_i)^2
+\frac{\kappa^2}{N^2\lambda^2}(\hat g^\sharp(x_i)-y_i)^2
\right). \nonumber
\end{align}
Then by \eqref{eq:gerror} and \eqref{eq:gaerror} we obtain
\begin{equation}
J_2 \le  \frac {4\kappa^4}{N^2\lambda^2 } \left( Q\Big(\sfrac{N\lambda}{N+1} ,\  N+1\Big)+\sigma^2 \right).
\label{eq:J2basic}
\end{equation}
For $J_3$, by Lemma \ref{swleaveoneout} again, we have
\begin{align}
& 2 \left|\left(\hat f ^\sharp _{(i)}(x_i) - \hat g^\sharp (x_i)\right)
\left( \hat g^\sharp (x_i) -y_i\right)\right|
\le 2 \kappa \| \hat f ^\sharp _{(i)} - \hat g^\sharp \|_K |\hat g^\sharp(x_i)-y_i| \nonumber \\[1ex]
\le\ &
2\kappa \left(\frac{\kappa}{N\lambda}|\hat g(x_i)-y_i|
+\frac{\kappa}{N\lambda}|\hat g^\sharp(x_i)-y_i|\right)
|\hat g^\sharp(x_i)-y_i| \nonumber\\[1ex]
\le\ &  \frac{\kappa^2}{N\lambda}(\hat g(x_i)-y_i)^2
+ \frac{3\kappa^2}{N\lambda}(\hat g^\sharp(x_i)-y_i)^2
\nonumber
\end{align}
By \eqref{eq:gerror} and \eqref{eq:gaerror} again, we have
\begin{align}
J_3 & \le \frac {4\kappa^2}{N \lambda}
 \left( Q\Big(\sfrac{N\lambda}{N+1} ,\  N+1\Big)+\sigma^2 \right).
 \label{eq:J3basic}
\end{align}
The desired error bound follows by combining the estimation for $J_1$, $J_2$ and $J_3$.
\end{proof}

Theorem \ref{thm:loo} part (i) and part (ii) now follow immediately from Proposition \ref{fhatsharpsample} and Lemma \ref{lem:Qbound}.

\medskip

If $r\ge 1$, the bound in Proposition \ref{fhatsharpsample} is still true but the estimation for $Q(\lambda, N)$ is not sufficient to prove sharp bounds for BCKRR. Instead, we will estimate $J_1$ in \eqref{eq:sharpdec} alternatively
 to obtain the following error bound for BCKRR, which together with by Lemma \ref{lem:Qsbound} allows to derive the sharp bounds in
 Theorem \ref{thm:loo} part (iii) and part (iv).

\begin{proposition}\label{prop:fhatsharp2}
	If $r\ge 1$, we have
	\begin{align*}
	\bE \left[ \| f^\sharp_{D,\lambda} -f^*\|_{\L2p}^2\right]
	\leq Q^\sharp\Big(\sfrac{N\lambda}{N+1} ,\  N+1\Big)
	+ \left( \frac{4\kappa^2}{N \lambda} + \frac{4\kappa^4}{N^2 \lambda^2} \right) \left\{ Q\Big(\sfrac{N \lambda}{N+1} ,\  N+1\Big) + \sigma^2\right\}.
	\end{align*}
\end{proposition}

\begin{proof}
Let $\tilde \lambda = \frac {N\lambda }{N+1}$ and $f_{\tilde \lambda}^\sharp$  be defined by \eqref{eq:flsharp} associated to the parameter $\tilde \lambda$. Since $\hat g^\sharp = f^\sharp_{\widetilde D, \tilde \lambda} ,$  we can regard $f_{\tilde \lambda}^\sharp$ as the sample limit of $\hat g.$
	Recall the error decomposition in \eqref{eq:sharpdec}.
	By \eqref{description}  and a similar process to the proof of Lemma \ref{lemma:Q}, we have
	\begin{align}
	&\bE\left[\frac 1 {N+1} \sum_{i=1}^{N+1} (\hat g^\sharp (x_i)-y_i)^2 \right]\nonumber \\
	&\le \frac N {N+1}\bE \left[\frac 1 {N} \sum_{i=1}^{N+1}  (\hat g^\sharp (x_i)-y_i)^2  + \lambda\|\hat g^\sharp - \hat g\|_K^2 \right]- \frac {N\lambda} {N+1}\bE \left[ \|\hat g^\sharp - \hat g\|_K^2 \right]  \nonumber \\
	& \le \bE \left[ \frac 1 {N+1} \sum_{i=1}^{N+1} (f_{\tilde \lambda} ^\sharp (x_i)-y_i)^2  \right]+\tilde\lambda  \bE\Big[\|f_{\tilde\lambda}^\sharp - \hat g\|_K^2-\|\hat g^\sharp - \hat g\|_K^2\Big] \nonumber \\
	& \le \|f_{\tilde\lambda}^\sharp - f^\ast\|_\L2p^2+\sigma^2+\tilde\lambda \bE \Big[\|f_{\tilde \lambda}^\sharp - \hat g\|_K^2-\|\hat g^\sharp - \hat g\|_K^2\Big].
	\nonumber
	\end{align}
	Hence we can bound $J_1$ as
	\begin{equation*}
	J_1 \le  \|f_{\tilde \lambda}^\sharp - f^\ast\|_\L2p^2+\tilde \lambda  \bE \Big[\|f_{\tilde \lambda}^\sharp - \hat g\|_K^2-\|\hat g^\sharp - \hat g\|_K^2\Big]
	= Q^\sharp\Big(\sfrac{N\lambda }{N+1}, N+1\Big).
	\end{equation*}
	Combining this with the estimations in \eqref{eq:J2basic} for $J_2$ and
	\eqref{eq:J3basic} for $J_3$ we obtain the desired bound.
\end{proof}

\section{Conclusions and discussions}
 \label{sec:conclusion}

 In this paper, we first proposed a general framework to analyze the performance of response weighted distributed regression algorithms.
 Then we conducted leave one analyses of KRR and BCKRR, which lead to sharp error bounds and capacity independent optimal rates for both approaches. The error bounds factored in the impact of unexplained variance of the response variable and hence are able to be used in combination with aforementioned framework to deduce sharp error bounds and optimal learning rates for distributed KRR and distributed BCKRR even when the kernel is imperfect in the sense that the true regression function does not lie in the associated reproducing kernel Hilbert space.

 Our analysis involves two interesting byproducts.
 The first one is the super fast rates for noise free learning.
 To our best knowledge, rates that are faster than $O(N^{-1})$
 have been observed for classification learning in some special situations
 \cite{smale2007learning} but have never been observed for regression learning in the literature. In this paper we first show that
 they are also possible for regression learning.

The second one is the two alternative perspectives to reformulate BCKRR.
They are not only  critical for us to analyze the performance of
BCKRR in this study, but also
 shed light on the design of bias corrected algorithms
 for other machine learning tasks. Recall that
 the original design of BCKRR in \cite{wu2017bias}
heavily depends on the explicit operator representation of KRR. In machine learning, most regression or classification algorithms are solved by
an iterative optimization process and no explicit analytic solution exists.
It is thus difficult or even impractical to characterize  the bias.
However, the idea of fitting residuals may apply to all regression problems
while recentering regularization can apply to all regularization schemes.
Regression learning usually adopts a loss function $L(|y-f(x)|)$ and minimizes the regularized empirical loss:
$$\hat f_L = \arg\min_{f\in\H_K} \left\{\frac 1 N \sum_{i=1}^N L(|y_i-f(x_i)|) + \lambda \|f\|_K^2\right\}.$$
Define the residuals $r_i = y_i-\hat f_L(x_i)$ and fit the residual by a function $\hat g_L$ such that
$$\hat g_L =  \arg\min_{g\in\H_K} \left\{\frac 1 N \sum_{i=1}^N L(|r_i-g(x_i)|) + \lambda \|g\|_K^2\right\}.$$
 The bias corrected estimator with respect to the loss $L$
 can then be defined as $\hat f^\sharp _L = \hat f_L + \hat g _L.$
 In binary classification one usually uses a loss function of form $L(yf(x))$ where
 $y\in\{1, -1\},$ for instance, the hingle loss in support vector machines
 or the logistic loss in logistic regression. Because $y_i$ are labels and residuals are meaningless,
 it is not appropriate to implement bias correction by fitting  residuals in binary classification.
 Instead, we can still use recentering regularization.
 Namely, suppose the regularized binary classification algorithm
 estimates a classifier $\hat f_L$ by
 $$\hat f_L = \arg\min_{f\in\H_K} \left\{\frac 1 N \sum_{i=1}^N L(y_i f(x_i)) + \lambda \|f\|_K^2\right\}.$$
We define the bias corrected classifier by
$$\hat f^\sharp _L = \arg\min _{f\in\H_K} \left\{\frac 1 N \sum_{i=1}^N L(y_i f(x_i)) + \lambda \|f - \hat f_L\|_K^2\right\}.$$
It would be interesting to study these bias corrected algorithms  in the future
and investigate their theoretical properties as well as empirical application domains.

%

\bibliographystyle{abbrv}
\bibliography{drloo}

\end{document}